\documentclass[12pt,letterpaper]{amsart}
\usepackage[utf8]{inputenc}
\usepackage[margin=1in,footskip=0.25in]{geometry}
\usepackage{amsmath, amssymb, amsthm, amsfonts}
\usepackage{commath}
\usepackage[table,xcdraw]{xcolor}
\usepackage{graphicx}
\usepackage{subfig}
\usepackage{float}
\usepackage{amsthm}
\usepackage{cite}
\usepackage{cleveref}

\newtheorem{theorem}{Theorem}[section]
\newtheorem*{theoremA}
{Theorem A}
\newtheorem*{theoremB}
{Theorem B}

\theoremstyle{definition}
\newtheorem{remark}[theorem]{Remark}

\newtheorem{lemma}[theorem]{Lemma}
\newtheorem{proposition}[theorem]{Proposition}

\theoremstyle{definition}
\newtheorem{definition}[theorem]{Definition}

\numberwithin{equation}{section}

\DeclareMathOperator*{\argmin}{arg\,min}

\newcommand{\volM}[2]{\text{vol}(B^M({#1}, {#2}))}
\newcommand{\BM}[2]{B^M({#1}, {#2})}

\newcommand{\volhat}[4]{\widehat{\text{vol}}[#1, #2](#3, #4)}
\newcommand{\avgrho}[2]{\mu_{\rho}({#1}, {#2})}
\newcommand{\avgrhoest}[3]{\widehat{\mu_{\rho}}[{#1}]({#2}, {#3})}

\newcommand{\Ndxr}[3]{N[{#1}]({#2}, {#3})}
\newcommand{\yhat}[4]{\hat{y}[#1, #2](#3, #4)}

\newcommand{\E}{\mathbb{E}}
\newcommand{\bR}{\mathbb{R}}
\newcommand{\rmin}{r_{\min}}
\newcommand{\rmax}{r_{\max}}
\newcommand{\rmink}{r_{\min, \,k}}
\newcommand{\rmaxk}{r_{\max, \,k}}
\newcommand{\Drk}{(\Delta r)_k}
\newcommand{\fracc}[1]{\frac{1}{#1}}
\newcommand{\var}{\text{var}}
\newcommand{\p}{\mathbb{P}}
\newcommand{\cov}{\mathrm{cov}}

\newcommand{\dkexact}{d_{X_k}}
\newcommand{\dkapprox}{\widehat{\dkexact}}
\newcommand{\Nothers}{N-1} % number of other points (besides a fixed x \in X)

\title{An intrinsic approach to scalar-curvature estimation for point clouds}
\author{Abigail Hickok and Andrew J. Blumberg}
\date{\today}

\begin{document}

\begin{abstract}
We introduce an intrinsic estimator for the scalar curvature of a data set presented as a finite metric space.  Our estimator depends only on the metric structure of the data and not on an embedding in $\bR^n$. We show that the estimator is consistent in the sense that for points sampled from a probability measure on a compact Riemannian manifold, the estimator converges to the scalar curvature as the number of points increases. To justify its use in applications, we show that the estimator is stable with respect to perturbations of the metric structure, e.g., noise in the sample or error estimating the intrinsic metric. We validate our estimator experimentally on synthetic data that is sampled from manifolds with specified curvature.
\end{abstract}

\maketitle

\section{Introduction}

A compact Riemannian manifold is a smooth manifold $M$ equipped with compatible choices of inner product for each tangent space $T_p M$. The presence of this structure equips $M$ with a metric (induced by the fact that the inner product gives a definition of the length of a tangent vector) and moreover lets us makes sense of various geometric notions on $M$---in particular, the notion of {\em curvature}.

Curvature, which measures the extent to which a Riemannian manifold deviates from being ``flat,'' is a generalization of the use of the second derivative to measure the extent to which a curve pulls away from the tangent line at a point. There are several different notions of curvature in Riemannian geometry.
The focus of this paper is \emph{scalar curvature}, which is a function $S\colon M \to \bR$ that quantifies the curvature at a point $x \in M$ by a number $S(x)$. On a surface, scalar curvature is proportional to Gaussian curvature, but in contrast to Gaussian curvature, scalar curvature is defined for higher-dimensional manifolds as well.

The purpose of this paper is to study the problem of estimating the scalar curvature of a manifold given a finite sample $X \subset M$ regarded as a finite metric space, which we assume consists of independent draws from 
 some (possibly nonuniform) probability density function $\rho\colon M \to \mathbb{R}_+$. Our estimator is based on the fact that scalar curvature at $x \in M$ characterizes the growth rate of the volume of a geodesic ball $\BM{x}{r}$ as $r$ increases. More precisely, as $r \to 0$, the scalar curvature $S(x)$ at $x \in M$ has the following relationship to geodesic ball volume:
\begin{equation}\label{eq:scalar_ball}
    \frac{\volM{x}{r}}{v_nr^n} = 1 - \frac{S(x)}{6(n+2)} r^2 + \mathcal{O}(r^4)\,,
\end{equation}
where $n$ is the dimension of the manifold, $v_n$ is the volume of a unit Euclidean $n$-ball, and $v_nr^n$ is the volume of a Euclidean $n$-ball of radius $r$.  We proceed by computing maximum-likelihood estimators for the volumes on the left side of equation~\eqref{eq:scalar_ball} and fitting a quadratic function to approximate $S(x)$. Specifically, our estimate is
\[
\hat{S}(x) = -6(n+2)\hat{C}(x),
\]
where $\hat{C}(x)$ is the quadratic coefficient of the fitted curve.

Our first main theorem shows that this scalar-curvature estimator is asymptotically stable. That is, small errors in distance measurement (e.g., from geodesic-distance estimation) cause only small errors in our scalar-curvature estimates. This result ensures that we can  accurately estimate scalar curvature in real-world data sets, which are invariably noisy.

\begin{theoremA}[Theorem~\ref{thm:stability}]
Let $M$ be a compact Riemannian manifold that is equipped with a probability measure $\rho$ that has full support.  Let $\{X_k\}_{k \in \mathbb{N}}$ be a sequence of finite samples, with $|X_k| \to \infty$, that are drawn from $M$ according to $\rho$ and equipped with metrics $d_k$ such that
\begin{equation*}
    \max_{(x, y) \in X_k \times X_k} |d_k(x, y) - d(x, y)| \to 0 \qquad \text{as } k \to \infty\,,
\end{equation*}
where $d(x, y)$ is the geodesic distance between $x$ and $y$.
Then for a suitable sequence of radius sequences and any sequence $\{x_k\}$ of points with $x_k \in X_k$, we have $\vert \hat{S}[d_k](x_k) -  \hat{S}[d](x_k) \vert \to 0$ in probability as $k \to \infty$.
\end{theoremA}

\begin{remark}
In fact, we can extract an effective stability bound from the proof of Theorem \ref{thm:stability}; the distance between curvature estimators is controlled by an explicit formula that involves an additive term based on the radius sequence and a term that is controlled by the discrepancy $\delta_k := \max_{(x, y) \in X_k \times X_k} |d_k(x, y) - d(x, y)|$. 
\end{remark}

We can then establish our second main theorem, which establishes that our scalar curvature estimate $\hat{S}(x)$ converges to the true scalar curvature $S(x)$ under certain asymptotic conditions as $|X_k| \to \infty$.

\begin{theoremB}[Theorem \ref{thm:convergence}]
Under the same hypotheses as the preceding theorem, for a suitable sequence of radius sequences,
$\vert \hat{S}[d_k](x_k) -  S(x_k) \vert \to 0$ in probability as $k \to \infty$, where $\{x_k\}$ is any sequence of points such that $x_k \in X_k$ for each $k$.
\end{theoremB}

We test our scalar-curvature estimator on point clouds that are sampled from several different manifolds (see Section \ref{sec:experiments}). Broadly, the experiments demonstrate that our method is accurate on manifolds of constant scalar curvature, especially on low-dimensional manifolds. 
Additionally, our method is robust with respect to additive isotropic noise.

The primary limitation of our method is that we typically cannot accurately estimate scalar curvature on regions of the manifold where the scalar curvature has high variation or where it achieves a minimum or maximum. For example, in our experiments on the torus and hyperboloid, we find that we cannot accurately estimate curvature near points where scalar curvature is minimized.  However, our results are qualitatively correct even when the pointwise estimates are inaccurate. On the torus and the hyperboloid, we correctly detect the scalar-curvature sign and relative changes in scalar curvatures across the surfaces.

\subsection*{Applications}

We estimate scalar curvature using only metric data (i.e., pairwise geodesic distances between points that are sampled from on a manifold). Consequently, our estimator can be applied to any finite metric space.  We expect that our estimator can be used to inform a choice of low-dimensional embedding space. For example, if one finds that the scalar curvature of a data set is negative (respectively, positive) everywhere, then this would suggest embedding the points into hyperbolic space (respectively, a sphere). In recent years, there has been much research on non-Euclidean embeddings, such as hyperbolic embeddings \cite{poincare, tradeoffs, hierarchy}.

Another application is the generalization of curvature to metric spaces that do not obviously come from manifolds.  A notable example is given by a network (specifically, a weighted graph), with the metric given by the shortest-path distance. This yields a definition of discrete scalar curvature that is defined on the vertices of a network. 

\subsection*{Related Work} 
To the best of our knowledge, there is only one other paper on scalar-curvature estimation for manifolds of any dimension. Sritharan et al. \cite{pnas} developed a different method to estimate scalar curvature by using the second fundamental form and the Gauss--Codazzi equation. However, their method requires an embedding of the points in Euclidean space. Furthermore, it is particularly sensitive to noise because it involves tangent-space estimation. In one experiment, in which points were sampled from a Klein bottle, their method did not recover the correct sign for the scalar curvature after only a small amount of Gaussian noise was added (standard deviation $\sigma = .01$).  In contrast, we are able to obtain higher accuracy on noisy data sets.

There are many methods to estimate Gaussian curvature from point clouds that are sampled from surfaces. Guerrero et al.~\cite{pcpnet} estimated Gaussian curvature by designing a neural network with a PointNet-inspired architecture~\cite{pointnet}. Topological data analysis can also be used to detect curvature. Bubenik et al.~\cite{bubenik} used persistent homology to classify point clouds by the Gaussian curvature of the constant-curvature surface from which they were sampled. Another Gaussian curvature estimation method is given by Cazals and Pouget~\cite{osculating}. However, none of these methods have a straightforward generalization to scalar curvature estimation.

Bhaskar et al.~\cite{smita} defined ``diffusion curvature,'' which is a new (unsigned) measure of local curvature for point clouds that are sampled from a manifold (with any dimension). Although diffusion curvature is not the same as Gaussian or scalar curvature, numerical experiments in~\cite{smita} suggest that it is correlated with Gaussian curvature. However, unlike Gaussian and scalar curvature, diffusion curvature is always positive, so it cannot be used to infer whether scalar curvature is positive or negative. By contrast, our scalar-curvature estimates are signed, so our method can be used to distinguish between regions of positive and negative curvature.

Chazal et al.~\cite{chazal} considered curvature measures (which are distinct from curvature). They showed that curvature measures can be estimated stably. However, as in~\cite{pnas}, their method requires an embedding of the points in Euclidean space. Moreover, their method is not feasible for point clouds in high dimensions because it requires computing and storing the boundaries and intersections of a set of balls in the ambient space; Chazal et al. implemented and tested their method only in $\mathbb{R}^3$. By contrast, the accuracy and computational complexity of our method in the present paper depends only on the intrinsic dimension of the manifold; it does not depend on the dimension of the ambient space.

Lastly, we note that there is an important relationship to discrete network curvature \cite{network_curvature_review}. Our scalar-curvature estimator can be applied to networks with the shortest-path metric. There are two other definitions of discrete scalar curvature for networks, both of which are defined as ``contractions'' of discrete Ricci curvature \cite{curvature_cancer, other_scalar}.  More precisely, the discrete scalar curvature at a node is defined to be the sum of its adjacent edges' discrete Ricci curvature.  Sandhu et al.~\cite{curvature_cancer} defined scalar curvature at a vertex as the contraction of Ollivier--Ricci curvature and Sreejith et al.~\cite{other_scalar} defined scalar curvature as the contraction of Forman--Ricci curvature. These are justified by the fact that scalar curvature is the trace of Ricci curvature. However, it has not been proven that either notion of discrete scalar curvature converges to the scalar curvature of the manifold when the network is a geometric network on a manifold.

\subsection*{Organization} We briefly review the basics of Riemannian geometry and scalar curvature in Section \ref{sec:background}. We discuss our method for estimating scalar curvature in Section \ref{sec:method}. We prove stability (Theorem \ref{thm:stability}) in Section \ref{sec:stability} and convergence (Theorem \ref{thm:convergence}) in Section \ref{sec:convergence}. Finally, we discuss our numerical experiments in Section \ref{sec:experiments}.

\subsection*{Acknowledgements}

We thank Yining Liu, Michael Mandell, and  Mason Porter for helpful conversations.

%%%%

\section{Background}\label{sec:background}

In this section, we briefly review relevant background on Riemannian geometry and scalar curvature. For a detailed treatment, we recommend any standard textbook, e.g.~\cite{petersen}.

A \emph{Riemannian manifold} $(M, g)$ is a smooth manifold $M$ with a Riemannian metric. A \emph{Riemannian metric} $g$ is an assignment of a symmetric bilinear form $g_x : T_x M \times T_x M \to \mathbb{R}$ for each point $x \in M$. The Riemannian metric defines an inner product on each tangent space $T_x M$. For example, in Euclidean space, the canonical Riemannian metric is the usual Euclidean inner product.

The Riemannian metric induces a metric on the manifold as follows. The \emph{norm} of a vector $v$ in the tangent space $T_xM$ at $x$ is $\norm{v} := g_x(v, v)^{1/2}$. The \emph{length} of a continuously differentiable path $\gamma: [a, b] \to M$ is $L(\gamma) := \int_a^b \norm{\gamma'(t)}\mathrm{d}t$. The \emph{geodesic distance} between points $x$ and $y$ in the same connected component of $M$ is
\begin{equation*}
    d_{M, g}(x, y) := \inf \{L(\gamma) \mid \gamma:[a, b] \to M \text{ is a $C^1$ path such that } \gamma(a) = x \text{ and }\gamma(b) = y\}.
\end{equation*}
The closed \emph{geodesic ball} centered at $x \in M$ with radius $r \geq 0$ is
\begin{equation*}
    \BM{x}{r} := \{y \in M \mid d_{M, g}(x, y) \leq r\}\,.
\end{equation*}

Scalar curvature characterizes the rate at which the volume of a geodesic ball $\BM{x}{r}$ grows as $r$ grows. Equation \eqref{eq:scalar_ball} gives the relationship between the scalar curvature $S(x)$ at $x \in M$ and the geodesic ball volume $\volM{x}{r}$ as $r \to 0$. For example, if $S(x)$ is negative (respectively, positive), then the volume of a small geodesic ball that is centered at $x$ tends to be larger (respectively, smaller) than the volume of an $n$-dimensional Euclidean ball of the same radius.

In this paper, we estimate scalar curvature from point clouds that are sampled randomly from $M$. Probability density functions $\rho\colon M \to \bR$ (for sampling points on $M$) are defined as follows. The induced \emph{Riemannian volume form} $dV$ is given in local coordinates by
\begin{equation*}
    dV = \sqrt{\vert g \vert} dx^1 \wedge \cdots \wedge dx^n\,.
\end{equation*}
Equivalently, the volume form $dV$ is defined to be the unique $n$-form on $M$ that equals $1$ on all positively-oriented orthnormal bases. The Riemannian volume form induces the \emph{Riemannian volume measure} $\mu$ on $M$; the measure of a Borel subset $A$ is
\begin{equation*}
    \mu(A) = \int_A dV\,.
\end{equation*}
A random point $x$ that is sampled from $M$ has \emph{probability density function} (pdf) $\rho: M \to \bR$ if
\begin{equation*}
    \p[x \in A] = \int_A \rho(y)dV
\end{equation*}
for all Borel subsets $A$ and $\p[M] = 1$. For example, the uniform pdf on $M$ is $\rho(x) \equiv \fracc{\text{vol}(M)}$, where $\text{vol}(M) := \mu(M)$ is the volume of $M$. For a more thorough introduction to statistics on Riemannian manifolds, see \cite{Rman_stats}.
%%%%

\section{Estimating scalar curvature via geodesic ball-volume estimation}\label{sec:method}

Suppose that we are given a \emph{distance matrix} $d_X$, which is an $N \times N$ matrix whose $(i, j)$th entry is the distance between $x_i$ and $x_j$ for points $x_i, x_j \in X$, where $X$ is a point cloud such that $|X| = N$. By a slight abuse of notation, we will write $d_X(x, y)$ to denote the distance between points $x \in X$ and $y \in X$. We assume that:
\begin{enumerate}
    \item $(X, d_X)$ is a metric subspace of an unknown Riemannian manifold $(M, g)$ of unknown dimension $n$ and
    \item $X$ is sampled randomly from an unknown probability density function $\rho :M \to \mathbb{R}_+$.
\end{enumerate}
Let $d$ denote the geodesic distance on $M$. Importantly, we do not assume that we are given coordinates for the point cloud $X$; we assume only that we have the distance matrix $d_X$. However, it is possible to begin with a point cloud $X$ (instead of its distance matrix $d_X$), from which one can estimate geodesic distances using, for example, the graph-approximation technique of Tenenbaum et al. \cite{isomap, graph_approx}.

We summarize our scalar-curvature estimation method in Figure \ref{fig:pipeline}. To estimate the scalar curvature $S(x)$ at a point $x \in X$, the idea of our approach is to estimate $\volM{x}{r}$ for a sequence of increasing $r$ and then estimate $S(x)$ by fitting a quadratic polynomial to the estimated ball-volume ratios $\volM{x}{r}/(v_nr^n)$.

\begin{figure}
    \centering
    \includegraphics[width = .9\textwidth]{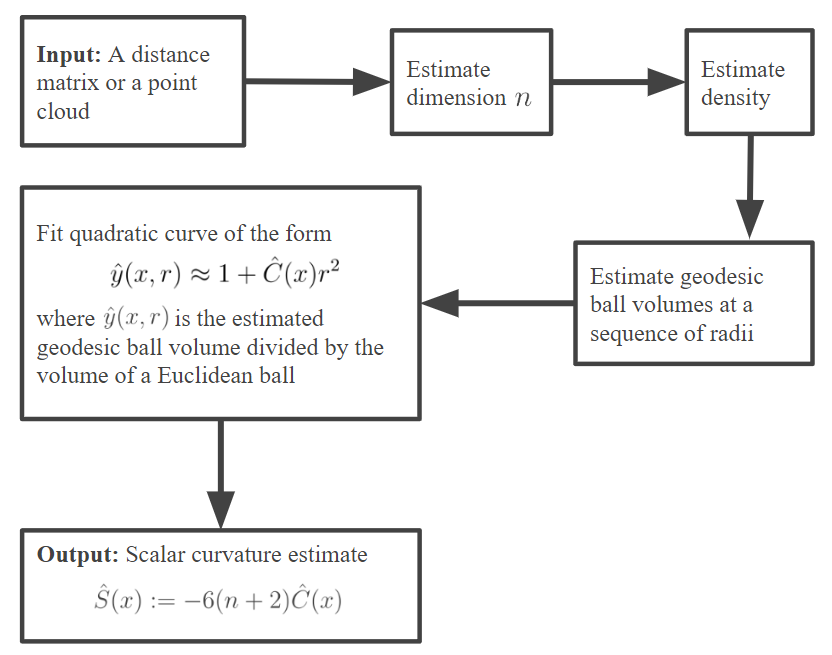}
    \caption{The pipeline for our scalar-curvature estimation method.}
    \label{fig:pipeline}
\end{figure}

\subsection{Maximum-likelihood estimator of ball volume}\label{sec:mle}
For a given radius $r$ and a point $x \in X$, we estimate $\volM{x}{r}$ as follows. Let $N$ be the number of points in $X$, and let $\Ndxr{d_X}{x}{r}$ denote the number of points in $\BM{x}{r} \cap (X \setminus \{x\})$. That is,
\begin{equation*}
    \Ndxr{d_X}{x}{r} := \vert \{y \in X\setminus \{x\} \mid d_X(x, y) \leq r\} \vert\,.
\end{equation*}
When the metric $d_X$ is clear from context, we omit it from the notation and write $N(x, r)$. Let $\avgrho{x}{r}$ denote the mean density within $\BM{x}{r}$. That is,
\begin{equation*}
    \avgrho{x}{r} := \fracc{\volM{x}{r}}\int_{z \in \BM{x}{r}} \rho(z)dV \,,
\end{equation*}
where $dV$ is the volume form on $M$. When $\rho(x) \equiv \rho$ is constant, $\avgrho{x}{r} = \rho$. In Section \ref{sec:empirical_mle}, we discuss a method to estimate $\avgrho{x}{r}$ empirically, without prior knowledge of $\volM{x}{r}$.

Our likelihood function for $\volM{x}{r}$ is
\begin{align*}
    L(v) &= \p[N(x, r) \mid \volM{x}{r} = v]\\ 
    &= {\Nothers \choose N(x, r)}\Big(\avgrho{x}{r} v\Big)^{N(x, r)}\Big(1 - \avgrho{x}{r} v\Big)^{\Nothers - N(x, r)}
\end{align*}
because the random variable $N(x, r)$ is a binomial random variable with $\Nothers$ trials and success probability $\avgrho{x}{r} \cdot \volM{x}{r}$. Solving $0 = L'(v)$, we find that the maximum-likelihood estimator is
\begin{equation}\label{eq:mle}
    v_* = \frac{N(x, r)}{(\Nothers)\avgrho{x}{r}}\,.
\end{equation}
The expectation of $v_*$ is
\begin{equation}\label{eq:volume_density}
    \E[v_*] = \frac{\E[N(x, r)]}{(\Nothers)\avgrho{x}{r}} = \volM{x}{r}.
\end{equation}

\subsection{Dimension estimation}

Our scalar-curvature estimation method requires an estimate $\hat{n} \in \mathbb{N}$ of the manifold dimension $n$; there are a wide variety of methods to do this --- see~\cite{dim_est_review} for a review. One method to estimate dimension is the maximum-likelihood method of Levina and Bickel~\cite{bickel}, which requires only the distance matrix $d_X$ as input. (See Section~\ref{sec:dim_est} for details.) When the distance matrix $d_X$ is not clear from context, we denote our dimension estimate by $\hat{n}[d_X]$. We assume that $\hat{n} = n$ in our theoretical results (Sections \ref{sec:stability} and \ref{sec:convergence}). In our numerical experiments (\Cref{sec:experiments}), we use~\cite{bickel} to calculate a dimension estimate $\hat{n}$.
%that equals the true dimension $n$ for each data set.

\subsection{Density estimation}
In Equation \eqref{eq:mle}, the mean density $\avgrho{x}{r}$ in the ball must be estimated empirically. To do so, we first empirically estimate the density at each point $z \in X$. One method for doing so is kernel density estimation (KDE) on a manifold \cite{kde_submanifold}, which requires only the distance matrix $d_X$ and an estimate $\hat{n}$ of the manifold dimension as input. 

\begin{remark}
We denote a choice of density estimator by $\hat{\rho}$, and we denote our density estimate at $z \in X$ by $\hat{\rho}[d_X, \hat{n}](z)$. For example, in our numerical experiments in Section \ref{sec:experiments}, the density estimator $\hat{\rho}$ is a kernel density estimator with either a Gaussian or biweight kernel. If $d_X$ and $\hat{n}$ are clear from context, we omit them and write $\hat{\rho}(z)$.
\end{remark}

After we compute our pointwise-density estimates $\hat{\rho}(z)$ for all $z \in X$, we calculate an estimate $\avgrhoest{\hat{\rho}}{x}{r}$ of the mean density $\avgrho{x}{r}$ within $\BM{x}{r}$. We define
\begin{equation}\label{eq:rhoest}
    \avgrhoest{\hat{\rho}}{x}{r} := \begin{cases}
        \Big(\fracc{N(x, r)}\sum_{z \in \BM{x}{r} \cap (X \setminus \{x\})} 1/\hat{\rho}(z)\Big)^{-1}\,, & N(x, r) > 0 \\
        \hat{\rho}(x) \,, & N(x, r) = 0\,.
    \end{cases}
\end{equation}
We write $\avgrhoest{\rho}{x}{r}$ when $\hat{\rho}(x) = \rho(x)$ for all $x \in X$.

Notably, our estimate $\avgrhoest{\hat{\rho}}{x}{r}$ is not the sample mean of $$\{\hat{\rho}(z) \mid z \in \BM{x}{r} \cap (X \setminus \{x\})\}\,.$$
The sample mean $\fracc{N(x, r)}\sum_z \hat{\rho}(z)$ is an overestimate of $\avgrho{x}{r}$ because points with high density are overrepresented in the sample $\BM{x}{r} \cap (X \setminus \{x\})$.

\begin{remark}
When $r$ is small, we have $\avgrho{x}{r} \approx \rho(x) \approx \hat{\rho}(x)$. Indeed, one can show that $\avgrho{x}{r} \to \rho(x)$ as $r \to 0$. However, in our numerical experiments in Section \ref{sec:experiments}, our scalar-curvature estimation method sometimes requires us to estimate $\avgrho{x}{r}$ when $r$ is not small. Informally, what we show in Lemma \ref{lem:mean_density_expectation} is that $1/\avgrhoest{\hat{\rho}}{x}{r}$ is a good approximation to $1/\avgrho{x}{r}$ even for large $r$. (Estimating the maximum-likelihood estimator of Equation \eqref{eq:mle} requires us to estimate the reciprocal $1/\avgrho{x}{r}$.) In our experiments (Section \ref{sec:experiments}), we observe significant empirical improvement from using Equation \eqref{eq:rhoest} to estimate $\avgrho{x}{r}$ instead of using $\hat{\rho}(x)$ to estimate $\avgrho{x}{r}$. This observation holds even when the data is uniformly sampled because $\avgrhoest{\hat{\rho}}{x}{r}$ averages the empirical densities (which may differ from the ground truth density) within the ball.
\end{remark}

\begin{lemma}\label{lem:mean_density_expectation}
If $X$ is a point cloud sampled from the pdf $\rho: M \to \bR_+$, then
\begin{equation*}
\E\Big[\fracc{\avgrhoest{\rho}{x}{r}} \, \Big\vert \, N(x, r) > 0 \Big] = \fracc{\avgrho{x}{r}}\,
\end{equation*}
for all $x \in X$ and $r > 0$.
\end{lemma}
\begin{proof}
If $r$ is sufficiently large so that $N(x, r) > 0$, then $\fracc{\avgrhoest{\rho}{x}{r}}$ is the sample mean of $1/\rho(z)$ for $z \in \BM{x}{r} \cap (X \setminus \{x\})$. Therefore, 
\begin{equation*}
\E\Big[\fracc{\avgrhoest{\rho}{x}{r}}\Big] = \E\Big[\fracc{\rho(z)}\Big]\,,
\end{equation*}
where $z$ is a point that is conditioned to lie in $\BM{x}{r}$. The pdf for $z$ is 
\begin{equation}\label{eq:random_ball_point}
    \psi(z) := \frac{\rho(z)}{\int_{w \in \BM{x}{r}} \rho(w)dV}\,.
\end{equation}
Therefore,
\begin{equation}\label{eq:E_1/rho}
    \E\Big[\fracc{\rho(z)}\Big] = \int_{z \in \BM{x}{r}} \fracc{\rho(z)}\psi(z)dV = \frac{\volM{x}{r}}{\int_{w \in \BM{x}{r}} \rho(w)dV} = \fracc{\avgrho{x}{r}}\,.
\end{equation}
\end{proof}

\subsection{Empirical approximation of the maximum-likelihood estimator}\label{sec:empirical_mle}

For a given $x \in X$ and radius $r > 0$, we define our estimate of $\volM{x}{r}$ to be
\begin{equation}\label{eq:volest}
    \volhat{d_X}{\hat{\rho}}{x}{r} := \frac{\Ndxr{d_X}{x}{r}}{(\Nothers) \avgrhoest{\hat{\rho}}{x}{r}}\,,
\end{equation}
where $\avgrhoest{\hat{\rho}}{x}{r}$ is defined as in Eq.~\eqref{eq:rhoest}. We write $\volhat{d_X}{\rho}{x}{r}$ when $\hat{\rho}(x) = \rho(x)$ for all $x \in X$. 

The quantity $\volhat{d_X}{\hat{\rho}}{x}{r}$ is an approximation of the true maximum-likelihood estimator $v_*$ (defined in equation~\eqref{eq:mle}). An equivalent formula for $\volhat{d_X}{\hat{\rho}}{x}{r}$ is
\begin{equation}\label{eq:volest_alt}
    \volhat{d_X}{\hat{\rho}}{x}{r} = \frac{\sum_{z \in \BM{x}{r} \cap (X \cap \{x\})} \fracc{\hat{\rho}(z)}}{(\Nothers)}\,.
\end{equation}

\begin{lemma}\label{lem:E_volhat}
    If $X$ is a finite point cloud that is sampled from the pdf $\rho: M \to \mathbb{R}_+$, then
    \begin{equation*}
        \E\Big[\volhat{d_X}{\rho}{x}{r} \Big] = \volM{x}{r}
    \end{equation*}
    for all $x \in X$ and $r > 0$.
\end{lemma}
\begin{proof}
    Let $N$ be the number of points in $X$. For all $k \in \{0, \ldots, \Nothers\}$,
    \begin{equation}\label{eq:Evolhat}
        \E\Big[\volhat{d_X}{\rho}{x}{r} \Big\vert N(x, r) = k \Big] = \frac{k}{\Nothers} \E[1/\rho(z)]
    \end{equation}
    by equation~\eqref{eq:volest_alt}, where $z$ is randomly drawn according to the pdf $\psi(z)$ defined in equation~\eqref{eq:random_ball_point}. Substituting equation~\eqref{eq:E_1/rho} into equation~\eqref{eq:Evolhat} yields
    \begin{equation}\label{eq:E_volhat_condition_Nxr}
        \E\Big[\volhat{d_X}{\rho}{x}{r} \Big\vert N(x, r) = k \Big] = \frac{k}{(\Nothers)\avgrho{x}{r}}
    \end{equation}
    for all $k \in \{0, \ldots, \Nothers\}$. Therefore,
    \begin{align*}
        \E\Big[\volhat{d_X}{\rho}{x}{r}\Big] &= \sum_{k=0}^{\Nothers} \E\Big[\volhat{d_X}{\rho}{x}{r}\Big\vert N(x, r) = k \Big] \cdot \p[N(x, r) = k] \\
        &= \fracc{(\Nothers) \avgrho{x}{r}}\sum_{k=0}^{\Nothers} k \cdot \p[N(x, r) = k] \\
        &= \frac{\E[N(x, r)]}{(\Nothers) \avgrho{x}{r}} \\
        &= \volM{x}{r}\,.
    \end{align*}
\end{proof}

\begin{lemma}\label{lem:var_volhat_rate}
    Let $X$ be a point cloud that consists of $N$ points that are sampled from the pdf $\rho : M \to \mathbb{R}_+$. If $M$ is compact, then there is a constant $A > 0$ that only depends on $\rho$ and the Riemannian metric of $M$ and satisfies
    \begin{equation*}
        \var(\volhat{d_X}{\rho}{x}{r})  \leq Ar^n/N
    \end{equation*}
	for sufficiently large $N$, sufficiently small $r$, and all $x \in X$.
\end{lemma}
\begin{proof}
    By Lemma \ref{lem:var_volhat},
	\begin{equation}\label{eq:var_bound_0}
		  \var(\volhat{d_X}{\rho}{x}{r}) = \frac{\var(1/\rho(z)) \cdot \avgrho{x}{r} \cdot \volM{x}{r}}{(\Nothers)} + \frac{\var N(x, r)}{(\Nothers)^2 \avgrho{x}{r}^2}\,,
	\end{equation}
    where $z \in \BM{x}{r}$ is a point chosen randomly from the pdf $\psi(z)$ defined as in Eq. \eqref{eq:random_ball_point} and 
    \begin{equation*}
		\var(N(x, r)) = (\Nothers)\avgrho{x}{r}\volM{x}{r}(1 - \avgrho{x}{r}\volM{x}{r})\,.
    \end{equation*}

Now we bound $\var(1/\rho(z))$. Define
    \begin{equation}\label{eq:Adef}
        A(r) := \max_{x \in M,\, z \in \BM{x}{r}} |\rho(z) - \rho(x)|\,.
    \end{equation}
    The quantity $A(r)$ exists because $M$ is compact and $\rho$ is continuous. We note that $A(r) \to 0$ as $r \to 0$. For the remainder of the proof, we assume that $r$ is sufficiently small such that $A(r) \leq \min(\rho)/2$. Because $h(\rho) = 1/\rho^2$ is convex and monotonically decreasing for $\rho > 0$, we have
    \begin{equation*}
        \Big\vert \fracc{\rho(z)^2} - \fracc{\rho(x)^2}\Big\vert \leq |h'(\min\{\rho(z), \rho(x)\})| \cdot A(r) \leq \frac{2A(r)}{(\rho(x) - A(r))^3} \leq \frac{16A(r)}{\min(\rho)^3}
    \end{equation*}
    for all $z \in \BM{x}{r}$. Therefore,
    \begin{equation*}
        \Big\vert \E[1/\rho(z)^2] - 1/\rho(x)^2\Big\vert \leq \frac{16A(r)}{\min(\rho)^3}\,.
    \end{equation*}
    Similarly,
    \begin{equation*}
         \Big\vert \fracc{\avgrho{x}{r}^2} - \fracc{\rho(x)^2}\Big\vert \leq \frac{16A(r)}{\min(\rho)^3}
    \end{equation*}
    because $|\avgrho{x}{r} - \rho(x)| \leq A(r)$. Therefore,
    \begin{align}
        \var(1/\rho(z)) &= \Big\vert \E\Big[\fracc{\rho(z)^2}\Big] - \fracc{\avgrho{x}{r}^2}\Big\vert \notag \\
        &\leq \Big\vert \E\Big[\fracc{\rho(z)^2}\Big] - \fracc{\rho(x)^2}\Big\vert + \Big\vert \fracc{\avgrho{x}{r}^2} - \fracc{\rho(x)^2}\Big\vert \notag \\
        &\leq \frac{32A(r)}{\min(\rho)^3}\,,\label{eq:var_1/rho}
    \end{align}
which implies that
\begin{equation*}\label{eq:var_bound_step1}
        \var(1/\rho(z)) \cdot \frac{\avgrho{x}{r} \cdot \volM{x}{r}}{(\Nothers)} \leq \frac{32\cdot A(r) \max(\rho)\cdot \volM{x}{r}}{\min(\rho)^3 (\Nothers)}
    \end{equation*}
for sufficiently small $r$ such that $A(r) < \min(\rho)/2$. By Lemma \ref{lem:small_ball_volume}, there is a constant $B' > 0$ such that $\volM{x}{r} \leq B' r^n$ for all $x$ and sufficiently small $r$. Additionally, we have $A(r) < 1$ for sufficiently small $r$, so
\begin{equation}
    \var(1/\rho(z)) \cdot \frac{\avgrho{x}{r} \cdot \volM{x}{r}}{(\Nothers)} \leq \frac{32B\cdot \max(\rho)}{\min(\rho)^3} \cdot \frac{r^n}{N} \label{eq:var_bound_1}
\end{equation}
for sufficiently large $N$, sufficiently small $r$, and some constant $B > 0$. Lastly, we bound $\frac{\var N(x, r)}{(\Nothers)^2 \avgrho{x}{r}^2}$. We have
\begin{align}
    \frac{\var N(x, r)}{(\Nothers)^2 \avgrho{x}{r}^2} &= \frac{\volM{x}{r}(1 - \avgrho{x}{r}\volM{x}{r}}{(\Nothers) \avgrho{x}{r}} \notag \\
    &\leq \frac{\volM{x}{r}}{(\Nothers) \avgrho{x}{r}} \notag \\
    &\leq \frac{B}{\min(\rho)} \cdot \frac{r^n}{N} \label{eq:var_bound_2} \,,
\end{align}
where $B > 0$ is the same constant as earlier in the proof. Substituting equations~\eqref{eq:var_bound_1} and \eqref{eq:var_bound_2} into equation~\eqref{eq:var_bound_0} completes the proof.
\end{proof}

\subsection{Fitting a quadratic curve}\label{sec:quad}
For radius $r > 0$, let $y(x, r)$ and $\hat{y}[d_X, \hat{\rho}, \hat{n}](x, r)$ denote the actual and estimated ball-volume ratios, respectively, for a ball of radius $r$ that is centered at a fixed $x \in X$. That is, we define
\begin{align}
    y(x, r) &:= \frac{\volM{x}{r}}{v_nr^n}\,, \label{eq:true_ratio} \\
    \hat{y}[d_X, \hat{\rho}, \hat{n}](x, r) &:= \frac{\volhat{d_X}{\hat{\rho}}{x}{r}}{v_{\hat{n}} r^{\hat{n}}} \label{eq:est_ratio}\,,
\end{align}
where $\volhat{d_X}{\hat{\rho}}{x}{r}$ is defined as in Eq. \eqref{eq:volest}. When $\hat{\rho}(x) = \rho(x)$ for all $x \in X$, we write $\hat{y}[d_X, \rho, \hat{n}](x, r)$. When $d_X$, $\hat{\rho}$, or $\hat{n}$ are clear from context, we omit them from our notation.

Let $\rmin$ and $\rmax$, respectively, be the minimum and maximum ball radius that we consider, where $0 \leq \rmin < \rmax$. These are hyperparameters that must be set by a user. Let $r_0:= \rmin < r_1 < \cdots < r_m := \rmax$ be a monotonically increasing sequence, which is also set by a user. These are the radius values at which we estimate geodesic ball volumes by empirically approximating the maximum-likelihood estimator, as in Section \ref{sec:empirical_mle}. We allow any choice of sequence $\{r_i\}_{j=1}^m$, although we study only two possible choices in this paper:
\begin{enumerate}
    \item {\bf Equal spacing:} The sequence is evenly spaced with spacing $\Delta r$. This is the choice that we make in Theorems \ref{thm:stability} and \ref{thm:convergence}.
    \item {\bf Nearest-neighbor distance:} In our numerical experiments, we allow $\{r_j\}$ to depend on $x$ and set $r_j$ to be equal to the distance from $x$ to its $j$th nearest neighbor.
\end{enumerate}

We define $C(x)$ to be the coefficient such that $1 + C(x)r^2$ is the ``best-fit'' quadratic curve to the curve $y(x, r)$for $r \in [\rmin, \rmax]$. More precisely, we define
\begin{equation*}
    C(x) := \argmin_{c \in \mathbb{R}} \norm{y(x, r) - (1 + cr^2)}_{L^2([\rmin, \rmax])}\,.
\end{equation*}
It is standard that
\begin{equation*}
	C(x) = \frac{\int_{\rmin}^{\rmax} r^2[y(x, r) - 1]dr}{\fracc{5}(\rmax^5 - \rmin^5)}\,.
\end{equation*}
We define
\begin{equation}\label{eq:Chat}
     \hat{C}[d_X, \hat{\rho}, \hat{n}](x) := \frac{\sum_{i=1}^m r_i^2 (\hat{y}[d_X, \hat{\rho}, \hat{n}](x, r_i) - 1)(r_i - r_{i-1})}{\fracc{5}(\rmax^5 - \rmin^5)}\end{equation}
to be an estimate of $C(x)$. We omit $d_X$, $\hat{\rho}$, and $\hat{n}$ from our notation when they are clear from context.

\subsection{Our scalar curvature estimate}
Putting together Sections \ref{sec:mle}--\ref{sec:quad}, we estimate scalar curvature.

\begin{definition}
Let $d_X$ be a distance matrix, let $\hat{\rho}$ be a density estimator, and let $\hat{n}$ be a dimension estimate. Given hyperparameters $\rmin \geq 0$ (the minimum ball radius that we consider), $\rmax > \rmin$ (the maximum ball radius that we consider), and $\{r_j\}_{j=0}^m$ (the sequence of ball radii that we consider, where $r_0 = \rmin$ and $r_m =\rmax$), our estimate of the scalar curvature at $x$ is
\begin{equation*}
    \hat{S}[d_X, \hat{\rho}, \hat{n}](x) := -6(\hat{n} +2)\hat{C}[d_X, \hat{\rho}, \hat{n}](x)\,,
\end{equation*}
where $\hat{C}[d_X, \hat{\rho}, \hat{n}](x)$ is defined in equation~\eqref{eq:Chat}.
\end{definition}
When the distance matrix $d_X$, density estimator $\hat{\rho}$, and dimension estimate $\hat{n}$ are clear from context, we omit them and write $\hat{S}(x)$.

\subsection{Computational complexity}
In our numerical experiments (Section \ref{sec:experiments}), we find that setting $r_i$ equal to the distance to the $i$th nearest neighbor results in an estimate that is both accurate and computationally efficient. In this case,
\begin{equation*}
    \volhat{d_X}{\hat{\rho}}{x}{r_i} = \fracc{N-1} \sum_{j=1}^m \fracc{\hat{\rho}(z_j)}\,,
\end{equation*}
where $z_j \in X$ is the $j$th nearest neighbor of $x$. We precompute the pointwise density estimates $\hat{\rho}(z)$ for all $z \in X$. For every $x \in X$, we sort $\{d(x, z) \mid z \in X\}$ to compute its nearest neighbors $z_1, z_2, \ldots$ and its distance to those neighbors.  (For very large data sets approximate nearest-neighbor algorithms could be used.)  Given these quantities, the set $\{\volhat{d_X}{\hat{\rho}}{x}{r_i}\}_{i=1}^m$ can be computed in $\mathcal{O}(m)$ time for any $m$ because
\begin{equation*}
    \volhat{d_X}{\hat{\rho}}{x}{r_{i+1}} = \volhat{d_X}{\hat{\rho}}{x}{r_i} + \fracc{(N-1)\cdot \hat{\rho}(z_{i+1})}\,.
\end{equation*}

%%%%

\section{Stability}\label{sec:stability}

Most real-world data sets have errors and/or noise, which means that the given distances $d_X$ will differ from the true geodesic distances $d$.  Moreover, when the geodesic distances are estimated from a point cloud, errors are expected even if there is no noise in the data (i.e., the point cloud) itself. Density estimation introduces additional errors. Theorem~\ref{thm:stability} below says that our scalar curvature estimate $\hat{S}$ is stable with respect to errors in estimates of the metric and the density. This allows us to accurately estimate scalar curvature in real-world data or in synthetic point-cloud data in which distances are estimated.

Throughout this section, we consider a compact $n$-dimensional Riemannian manifold $M$ with geodesic distance $d$ and a sequence $\{X_k\}_{k=1}^{\infty}$ of point clouds that are sampled randomly from a pdf $\rho\colon M \to \mathbb{R}_+$. We assume that $|X_k| \to \infty$ as $k \to \infty$. Let $\dkexact$ denote the geodesic distance matrix for $X_k$. By a slight abuse of notation, let $\dkexact(x, y)$ denote the geodesic distance between points $x \in X_k$ and $y \in X_k$. We also consider sequences $\{\rmink\}_{k=1}^{\infty}$, $\{\rmaxk\}_{k=1}^{\infty}$, and $\{\Drk\}_{k=1}^{\infty}$ of hyperparameter values. The $k$th radius sequence that we consider is $\{r_{j, k}\}_{j=0}^{m_k}$, where $r_{j, k} := \rmink + j \Drk$. When $k$ is clear from context, we omit it and write $r_j$ instead of $r_{j, k}$. We require that
\begin{enumerate}
    \item $0 < \rmink < \rmaxk$ for all $k$,
    \item the number $m_k := (\rmaxk - \rmink)/\Drk$ of radial steps is a positive integer for all $k$, and
    \item $\rmink \to 0$\,, $\rmaxk \to 0$\,,  and $\Drk \to 0$ as $k \to \infty$.
\end{enumerate}

\begin{theorem}[Stability]\label{thm:stability}
For each $k$, suppose that $\dkapprox$ is a metric on $X_k$ such that
\begin{equation*}
    \delta_k := \max_{x, x' \in X_k} \vert \dkapprox(x, x') - d(x, x') \vert \to 0 \qquad \text{as } k \to \infty\,.
\end{equation*}
Suppose that $\hat{\rho}$ is a density estimator such that
\begin{equation*}
    \eta_k := \max_{x \in X_k} \Big\vert \hat{\rho}[\dkapprox](x) - \rho(x) \Big\vert \to 0 \qquad \text{as } k \to \infty\,,
\end{equation*}
and suppose that $\hat{n}[\dkapprox] = \hat{n}[\dkexact] = n$ for sufficiently large $k$.
If the hyperparameter value sequences satisfy
\begin{enumerate}
    \item $\max_j \frac{A(2r_j)}{r_j^n \rmaxk^2} \to 0$ as $k \to \infty$ (where $A(r)$ is defined in equation~\eqref{eq:Adef})\,,
    \item $\eta_k/(\rmink + \Drk)^{n+2/3} \to 0$ as $k \to \infty$\,,
    \item $\rmink + \Drk > \delta_k$ for sufficiently large $k$\,,
    \item $|X_k|\Drk(\rmink + \Drk - \delta_k)^n \to \infty$ as $k \to \infty$\,,
    \item $\rmink/\rmaxk^3 \to 0$ as $k \to \infty$\,,
    \item $(\Drk + \delta_k)/\rmaxk^3 \to 0$ as $k \to \infty$, and
    \item $(\Drk + \delta_k)/[(\rmink + \Drk - \delta_k)^{n+1}\rmaxk^2] \to 0$ as $k \to \infty$\,,
\end{enumerate}
then $\vert \hat{S}[\dkapprox, \hat{\rho}, \hat{n}](x_k) -  \hat{S}[\dkexact, \rho, \hat{n}](x_k) \vert \to 0$ in probability as $k \to \infty$, where $\{x_k\}$ is any sequence of points such that $x_k \in X_k$.
\end{theorem}
\begin{remark}\label{rmk:simple}
The conditions above on the hyperparameter value sequences are complex. The following is a set of simpler conditions that collectively imply the conditions of Theorem \ref{thm:stability}:
\begin{enumerate}
    \item $\frac{A(r)}{r^{n+2}} \to 0$ as $r \to 0$\,,
    \item $\eta_k/\rmink^{n+2/3} \to 0$ as $k \to \infty$\,,
    \item $\delta_k = \mathcal{O}(\rmink^{n+2})$ as $k \to \infty$\,,
    \item $|X_k|\Drk (\rmink^n) \to \infty$ as $k \to \infty$\,,
    \item $(\rmink)/\rmaxk^3 \to 0$ as $k \to \infty$\,, and
    \item $\Drk/\rmink^{n + 5/3} \to 0$ as $k \to \infty$\,.
\end{enumerate}
\end{remark}
\begin{proof}[Proof of Theorem \ref{thm:stability}]
For any $x \in X_k$, we have
\begin{equation*}
\Big|\hat{S}[\dkapprox, \hat{\rho}](x) - \hat{S}[\dkexact, \rho](x)\Big| = 6(n+2)\Big\vert\hat{C}[\dkapprox, \hat{\rho}](x) - \hat{C}[\dkexact, \rho](x)\Big\vert\,.
\end{equation*}
The theorem follows from Lemma \ref{lem:Cstability}, which shows that
\begin{equation*}
\vert\hat{C}[\dkapprox, \hat{\rho}](x_k) - \hat{C}[\dkexact, \rho](x_k)\vert \to 0
\end{equation*}
in probability as $k \to \infty$.
\end{proof}

%%%%

\section{Convergence}\label{sec:convergence}

In this section, we show that as the number of samples increases our estimator converges to the underlying scalar curvature of $M$.  Informally, what we show in Theorem~\ref{thm:convergence} is that (1) as the number of points increases, (2) as our given metric data becomes more accurate, and (3) as our density estimations become more accurate, our scalar curvature estimate $\hat{S}(x)$ converges to the true scalar curvature $S(x)$. Throughout this section, the symbols $M$, $d$, $n$, $\{X_k\}_{k=1}^{\infty}$, $\dkexact$, $\rho$, $\{\rmink\}_{k=1}^{\infty}$, $\{\rmaxk\}_{k=1}^{\infty}$, $\{\Drk\}_{k=1}^{\infty}$, $m_k$ and $\{r_j\}_{j=1}^{m_k}$ are defined as in Section \ref{sec:stability}.

Theorem \ref{thm:convergence} is an immediate consequence of Theorem~\ref{thm:stability} (stability) above and Proposition~\ref{prop:convergence_exact} below; the latter states that if we are given perfect metric data and the exact density, then our scalar curvature estimate $\hat{S}(x)$ converges to $S(x)$ as the number of points increases. The challenge is that we must take $\rmaxk \to 0$ for equation~\eqref{eq:scalar_ball} to hold, but (as we show in Proposition~\ref{prop:MSE} below) the mean squared error of the ball-ratio estimate $\yhat{d_X}{\rho}{x}{r}$ grows as $\mathcal{O}(1/(Nr^n))$ as $r \to 0$, where $N$ is the number of points in the point cloud.

\begin{proposition}\label{prop:MSE}
Let $X$ be a point cloud that consists of $N$ points that are drawn from the pdf $\rho\colon M \to \mathbb{R}_+$. Then there is a constant $A > 0$ that only depends on $\rho$ and the Riemannian metric of $M$ such that
\begin{equation*}
    \text{MSE}(\yhat{d_X}{\rho}{x}{r}) = \var(\yhat{d_X}{\rho}{x}{r}) \leq \frac{A}{Nr^n}
\end{equation*}
for sufficiently large $N$, sufficiently small $r$, and all $x \in X$.
\end{proposition}
\begin{proof}
By Lemma \ref{lem:E_volhat},
 \begin{equation*}
     \text{MSE}(\yhat{d_X}{\rho}{x}{r}) = \var(\yhat{d_X}{\rho}{x}{r})\,.
 \end{equation*}
By Lemma \ref{lem:var_volhat_rate}, there is a constant $A' > 0$ such that
\begin{equation*}
    \var(\volhat{d_X}{\rho}{x}{r})  \leq \frac{A'r^n}{N}
\end{equation*}
for sufficiently large $N$, sufficiently small $r$, and all $x \in X$. Therefore,
\begin{equation*}
    \var(\yhat{d_X}{\rho}{x}{r}) = \frac{\var(\volhat{d_X}{\rho}{x}{r}}){v_n^2 r^{2n}} \leq \frac{A'}{v_n^2 Nr^n}
\end{equation*}
for sufficiently large $N$, sufficiently small $r$, and all $x \in X$.	
\end{proof}

\begin{proposition}\label{prop:convergence_exact}
Suppose that the estimated dimension $\hat{n}[\dkexact] = n$ for sufficiently large $k$. If the hyperparameter value sequences satisfy 
\begin{enumerate}
    \item $\Drk/\rmaxk^3 \to 0$ as $k \to \infty$\,,
    \item $|X_k|(\rmink + \Drk)^n \to \infty$ as $k \to \infty$\,, and
    \item $\rmink/\rmaxk^3 \to 0$ as $k \to \infty$\,,
\end{enumerate}
% Suppose that as $N \to \infty$, we have $\rmax \to 0$, $\frac{\rmaxk^3}{\Delta r} \to \infty$, $N(\rmink + \Delta r)^n \to \infty$, and $\rmin/\rmaxk^3 \to 0$. Then $\hat{S}[d, \rho](x) \to S(x) $ in probability as $N \to \infty$.
then $|\hat{S}[\dkexact, \rho, \hat{n}](x_k) \to S(x_k)| \to 0$ as $k \to \infty$, where $\{x_k\}$ is any sequence of points such that $x_k \in X_k$.
\end{proposition}
\begin{proof}
Let $x$ be any point in $X_k$. By Eq. \eqref{eq:scalar_ball}, we have
	\begin{equation*}
		C(x) = \frac{\int_{\rmink}^{\rmaxk} \Big[- \frac{S(x)}{6(n+2)} r^4 + \mathcal{O}(r^6)\Big]dr}{\fracc{5}(\rmaxk^5 - \rmink^5)} = - \frac{S(x)}{6(n+2)} + \mathcal{O}(\rmaxk^2)\,.
	\end{equation*}
The absolute difference $\vert \hat{S}[\dkexact, \rho](x) - S(x) \vert$ is
\begin{equation*}
	\vert \hat{S}[\dkexact, \rho](x) - S(x) \vert = 6(n+2)\vert \hat{C}[\dkexact, \rho](x) - C(x)\vert + \mathcal{O}(\rmaxk^2) \,.
\end{equation*}
 Applying Lemma \ref{lem:Cconverge}, which controls $\vert \hat{C}[\dkexact, \rho](x) - C(x)\vert$, yields the desired result.
\end{proof}

Theorem~\ref{thm:convergence} now follows from Theorem~\ref{thm:stability} and Proposition~\ref{prop:convergence_exact}.

\begin{theorem}\label{thm:convergence}
For each $k$, suppose that $\dkapprox$ is a metric on $X_k$ such that
\begin{equation*}
    \delta_k := \max_{x, x' \in X_k} \vert \dkapprox(x, x') - d(x, x') \vert \to 0 \qquad \text{as } k \to \infty\,.
\end{equation*}
Suppose that $\hat{\rho}$ is a density estimator such that
\begin{equation*}
    \eta_k := \max_{x \in X_k} \Big\vert \hat{\rho}[\dkapprox](x) - \rho(x) \Big\vert \to 0 \qquad \text{as } k \to \infty\,.
\end{equation*}
Suppose that $\hat{n}[\dkapprox] = \hat{n}[\dkexact] = n$ for sufficiently large $k$.
If the hyperparameter value sequences satisfy
\begin{enumerate}
    \item $\max_j \frac{A(2r_j)}{r_j^n \rmaxk^2} \to 0$ as $k \to \infty$ (where $A(r)$ is defined in equation~\eqref{eq:Adef})\,,
    \item $\eta_k/(\rmink + \Drk)^{n+2/3} \to 0$ as $k \to \infty$\,,
    \item $\rmink + \Drk > \delta_k$ for sufficiently large $k$\,,
    \item $|X_k|\Drk(\rmink + \Drk - \delta_k)^n \to \infty$ as $k \to \infty$\,,
    \item $\rmink/\rmaxk^3 \to 0$ as $k \to \infty$\,,
    \item $(\Drk + \delta_k)/\rmaxk^3 \to 0$ as $k \to \infty$, and
    \item $(\Drk + \delta_k)/[(\rmink + \Drk - \delta_k)^{n+1}\rmaxk^2] \to 0$ as $k \to \infty$
\end{enumerate}
then $\vert \hat{S}[\dkapprox, \hat{\rho}, \hat{n}](x_k) -  S(x_k) \vert \to 0$ in probability as $k \to \infty$, where $\{x_k\}$ is any sequence of points such that $x_k \in X_k$.
\end{theorem}
\begin{remark}
The simpler set of conditions from Remark \ref{rmk:simple} collectively implies the conditions of Theorem \ref{thm:convergence}.
\end{remark}

%%%%

\section{Numerical Experiments}\label{sec:experiments}

\subsection{Data sets}

We generate synthetic data by sampling uniformly at random from manifolds with known scalar curvature.

First, we sample $N = 10^4$ points each from three constant-curvature surfaces:
\begin{enumerate}
    \item A disk in the Euclidean plane with radius $2$. The scalar curvature is $S(x) \equiv 0$.
    \item A unit $2$-sphere. The scalar curvature is $S(x) \equiv 2$.
    \item A disk in the hyperbolic plane with hyperbolic radius $2$. The scalar curvature is $S(x) \equiv -2$\,.
\end{enumerate}
For the last of these, we use the Poincar\'e disk model. Notably, the points that we sample from the hyperbolic plane are not embedded in Euclidean space, which means that it is not possible to use the scalar-curvature estimation method of \cite{pnas}. To avoid boundary effects, we only estimate curvature at points within the unit disk in the Euclidean sample and within hyperbolic radius $1$ in the hyperbolic sample. Additionally, we sample point clouds from $S^2$ with noise. For $\sigma \in \{.001, .003, .01, .03\}$, we sample $N = 10^4$ points from $S^2$ and add isotropic Gaussian noise with standard deviation $\sigma$.

Next, we sample point clouds from several other manifolds. We sample $N = 10^4$ points each from the higher-dimensional unit spheres $S^n$ for $n = 3, 5,$ and $7$. Lastly, we sample one point cloud each from two surfaces with non-constant scalar curvature:

\begin{enumerate}
    \item A $2$-torus. We sample $N = 10^4$ points from a $2$-torus with parameters $r = 1$, $R = 2$.
    \item A one-sheet hyperboloid. The points $(x, y, z) \in \bR^3$ are given by the equations
    \begin{align*}
        x &= 2\sqrt{1 + u^2} \cos(\theta)\,, \\
        y &= 2\sqrt{1 + u^2} \sin(\theta)\,,\\
        z &= u
    \end{align*}
    for $u \in \bR$ and $\theta \in [0, 2\pi)$. We sample points uniformly at random from the subset of the hyperboloid such that $|z| \leq 2$ until we have $N = 10^4$ points within the subset such that $|z| \leq 1$. To avoid boundary effects, we only estimate curvature at points on the hyperboloid such that $|z| \leq 1$.
\end{enumerate}

\subsection{Dimension estimation}\label{sec:dim_est}
To estimate dimension, we use the maximum-likelihood method of Levina and Bickel \cite{bickel}. Our estimate of the dimension of a point cloud $X$ is the nearest integer $\hat{n}$ to
\begin{equation*}
    \fracc{k_2 - k_1 + 1} \sum_{k = k_1}^{k_2} \hat{n}_k\,,
\end{equation*}
where $k_1$ and $k_2$ are hyperparameters and
\begin{align*}
    \hat{n}_k &:= \fracc{N} \sum_{i=1}^N \hat{n}_k(x_i)\,,\\
    \hat{n}_k(x_i) &:= \Big[ \fracc{k-1}\sum_{j=1}^{k-1} \log \Big( \frac{T_k(x_i)}{T_j(x_i)}\Big) \Big]^{-1}\,,
\end{align*}
where $T_j(x_i)$ is the distance from $x_i$ to its $j$th nearest neighbor in $X$. For all data sets, we set $k_1 = 20$ and calculate $\hat{n}$ for $k_2 \in \{30, \ldots, 100\}$. We obtain $\hat{n} = n$, where $n$ is the ground-truth dimension, for all data sets and all choices of $k_2$.

We make one modification to \cite{bickel}, which is that instead of using Euclidean distance to measure distances to nearest neighbors, as was done in \cite{bickel}, we use geodesic distance.\footnote{In cases where we can calculate both exact and estimated geodesic distances, we use both; otherwise, we use whichever is available. For $S^2$, $S^3$, $S^5$, $S^7$, and the Euclidean disk, we possess both exact and estimated geodesic distances. For the Poincar\'e disk, we have only exact geodesic distances. For all other data sets, we have only estimated geodesic distances.} This choice reduces overall computation time because computing geodesic nearest-neighbor distances is also part of our scalar-curvature estimation pipeline. In addition, using geodesic distance improves the accuracy of the approximations that were made in \cite{bickel} and allows us to estimate the dimension of our Poincar\'e-disk data, which is not embedded in Euclidean space.

\subsection{Density estimation}

We use kernel density estimation to obtain pointwise estimates of density, using the dimension estimates obtained in Section \ref{sec:dim_est}. We test two choices of kernel: (1) a Gaussian kernel because it is a very common choice for density estimation and (2) a  biweight kernel because it is compactly supported. As input, the kernel function takes geodesic distances (either exact or estimated), rather than Euclidean distances.

\subsection{Geodesic-distance estimation}

On the spheres, the Euclidean disk, the torus, and the hyperboloid, we estimate pairwise geodesic distances using the method of Tenenbaum et al. \cite{isomap, graph_approx}. For each point cloud, we construct the $k$-nearest neighbor graph $G$ with $k = 20$ for $n = 2$, with $k = 50$ for $n = 3$, with $k = 100$ for $n = 5$, and with $k = 200$ for $n = 7$. Edge weights are Euclidean distances. Our estimation of the geodesic distance between points $x_1$ and $x_2$ is the length of a shortest weighted path in $G$.

\subsection{Hyperparameter choices}
Our method requires a choice of minimum ball radius $\rmin$, maximum ball radius $\rmax$, and radius sequence $\{r_j\}_{j = 0}^m$ such that $r_0 = \rmin$ and $r_m = \rmax$. For a given point $x$ in a data set, we set $r_i$ equal to the distance from $x$ to its $i$th nearest neighbor (as measured by the given distance matrix $d_X$), for the subset of neighbors such that $\rmin \leq r_i \leq \rmax$. We set $\rmin = 0$ for all data sets.

Our choice of $\rmax$ differs across data sets because the scales and sampling densities are different in different data sets. For the spheres (including the point clouds with noise), we set $\rmax = \pi/2$. For the Euclidean and Poincar\'e disks, we set $\rmax = 1$. For the torus, we set $\rmax = \pi$. For the hyperboloid, we set $\rmax = 2$. These values were chosen to minimize the amount of noise in our curvature estimation results and to ensure that our geodesic balls $\BM{x}{r}$ do not intersect the boundary of the manifold $M$.

\subsection{Results}\label{sec:results}

First, we apply our method to our constant-curvature data sets. For the two surfaces that are embedded into Euclidean space ($S^2$ and the Euclidean disk), we test our method in two different ways. First, we use the exact geodesic distances for our distance matrix. Second, we estimate geodesic distances from the point clouds. In Figure \ref{fig:constant}, we show our results.

\begin{figure}
    \centering
    \subfloat[]{\includegraphics[width = .45\textwidth]{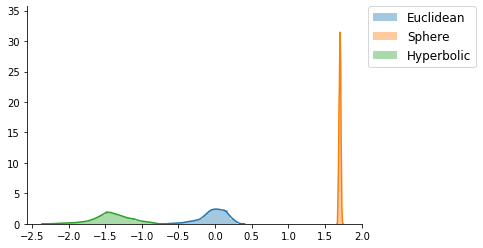}}
    \hspace{5mm}
    \subfloat[]{\includegraphics[width = .45\textwidth]{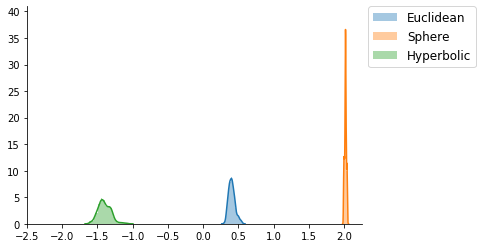}} \\
    \subfloat[]{\includegraphics[width = .45\textwidth]{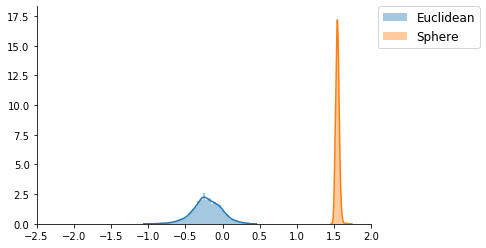}}
    \hspace{5mm}
    \subfloat[]{\includegraphics[width = .45\textwidth]{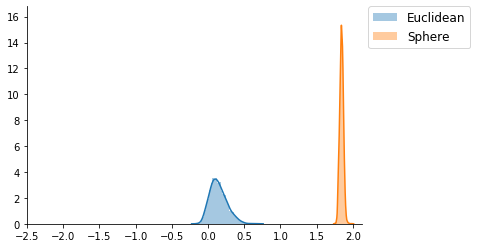}}
    \caption{Histograms for our scalar-curvature estimates on three surfaces of constant curvature, given (A--B) exact geodesic distances and (C--D) point clouds, from which geodesic distances were estimated. In (A) and (C), we use a Gaussian kernel to estimate density, and in (B) and (D), we use a biweight kernel to estimate density. The ground-truth scalar curvatures values are $-2$ in the hyperbolic disk, $0$ in the Euclidean disk, and $2$ on the sphere.  (Note that we only have exact distances on the hyperbolic disk.)}
    \label{fig:constant}
\end{figure}

We next test our method on the point clouds that are sampled from higher-dimensional spheres. Again, we test our method in two scenarios: (1) given as input exact geodesic distances and (2) using estimated geodesic distances from the point clouds. In early experiments, we found that on the highest-dimensional spheres ($n \geq 5$), using a biweight kernel to estimate density led to significantly better performance than using a Gaussian kernel, so we use a biweight kernel for density estimation. In Figure \ref{fig:spheres}, we show our results. Unexpectedly, we find in Figure \ref{fig:spheres}(A) that scalar curvature is systematically underestimated (although still reasonably accurate) when we start with the exact geodesic distances. In both experiments, the accuracy of our estimates decreases as the dimension $n$ increases, but the performance is comparable to that in \cite{pnas}. The main reason that scalar curvature is more difficult for us to estimate in higher dimensions is that the mean squared error in our ball-ratio estimates increases exponentially in $n$ (see Proposition~\ref{prop:MSE}). Another reason is that the accuracy of geodesic-distance estimation decreases as $n$ increases and $N$ stays constant. Typically, the number of points $N$ must scale exponentially with $n$ to maintain the same ``resolution'' of the manifold, so it is unsurprising that our scalar curvature estimates become less accurate as $n$ increases for fixed $N$.
\begin{figure}
    \centering
    \subfloat[]{\includegraphics[width = .45\textwidth]{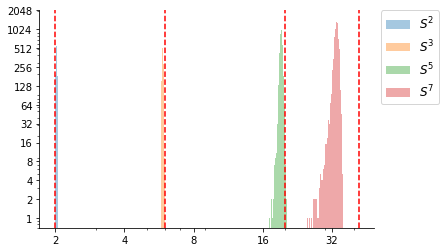}}
    \hspace{5mm}
    \subfloat[]{\includegraphics[width = .45\textwidth]{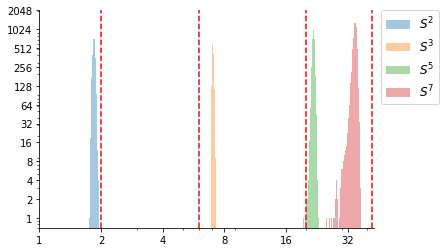}}
    \caption{Histograms for our scalar-curvature estimates on $S^n$ (for $n = 2, 3, 5, 7$) using (A) exact geodesic distances and (B) point clouds, from which geodesic distances were estimated. In (A) and (B), the histograms are plotted on a log--log scale. The ground-truth scalar curvature, which is indicated by the red dashed lines, is $S(x) \equiv n (n-1)$ for each $n$ and all $x \in S^n$.}\label{fig:spheres}
\end{figure}

To test our method on manifolds with non-constant scalar curvature, we apply our scalar-curvature estimator to our torus and hyperboloid data sets. On both surfaces, we find that using a Gaussian kernel for density estimation yields more accurate curvature estimates, so we use a Gaussian kernel. We show our results in Figure \ref{fig:nonconstant_surfaces}. On the torus, our estimator correctly distinguishes between regions of positive, negative, and zero scalar curvature. The estimates are accurate except near $\theta = \pi$, where scalar curvature is minimized. On the hyperboloid, our estimator correctly identifies the fact that scalar curvature is minimized (and negative) near $z = 0$ and increases as $z$ increases. As in the torus, the estimates are accurate except near $z = 0$, where scalar curvature is minimized.

\begin{figure}
    \centering
    \subfloat[]{\includegraphics[width = .4\textwidth]{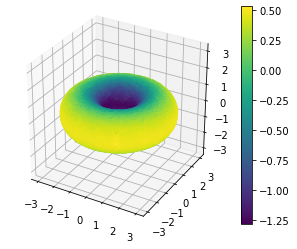}}
    \hspace{10mm}
    \subfloat[]{\includegraphics[width = .4\textwidth]{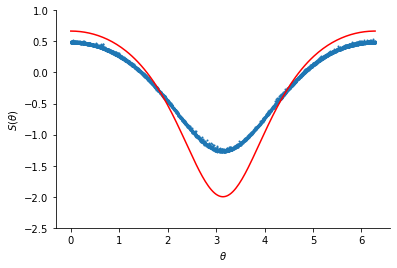}}
    \\
    \subfloat[]{\includegraphics[width = .4\textwidth]{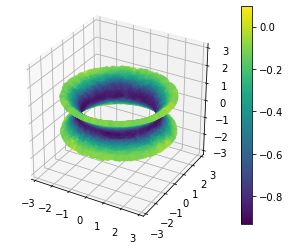}}
    \hspace{10mm}
    \subfloat[]{\includegraphics[width = .4\textwidth]{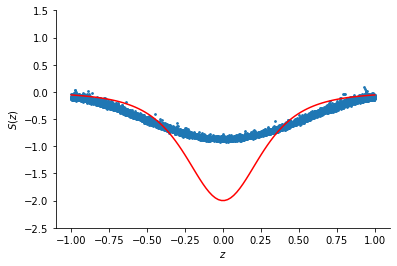}}
    \caption{(A) Scalar-curvature estimation on a torus. (B) Scalar curvature on the torus as a function of angle $\theta$. In red, we show the exact scalar curvature values; in blue, we show the estimated scalar curvature values.
    (C) Scalar-curvature estimation on a one-sheet hyperboloid. (D) Scalar curvature on the hyperboloid as a function of the $z$ coordinate. In red, we show the exact scalar curvature values; in blue, we show the estimated scalar curvature values.
    }
    \label{fig:nonconstant_surfaces}
\end{figure}

We investigate the stability of our estimator by estimating curvature on our noisy-sphere data sets. We show our results in Figure \ref{fig:noise}. In Figures \ref{fig:noise}(A) and (B), we show our results when we use Gaussian and biweight kernels, respectively, for density estimation and we input the estimated geodesic distances to the kernel. At the highest noise level (standard deviation $\sigma = .03$), our scalar curvature estimates have the wrong sign when we use a biweight kernel, but all other curvature estimates have the correct sign. In Figures \ref{fig:noise}(C) and (D), we test our estimator by inputting Euclidean distances to the kernel for density estimation. We find that performance is significantly improved, especially at the highest noise level ($\sigma = .03$). This suggests that if a point cloud has a high noise level, then one should input the Euclidean distances to the kernel instead of inputting the estimated geodesic distances, which may not be accurate enough.

\begin{figure}
    \centering
    \subfloat[]{\includegraphics[width = .45\textwidth]{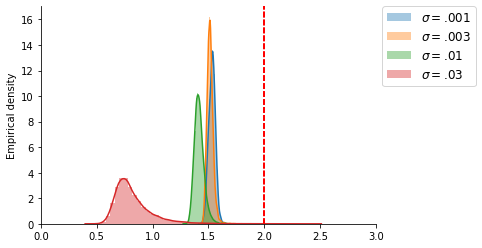}}
    \hspace{5mm}
    \subfloat[]{\includegraphics[width = .45\textwidth]{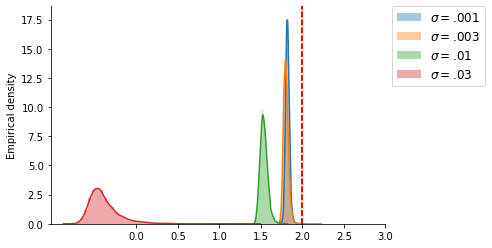}}\\
    \subfloat[]{\includegraphics[width = .45\textwidth]{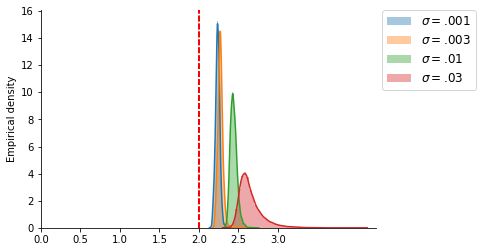}}
    \hspace{5mm}
    \subfloat[]{\includegraphics[width = .45\textwidth]{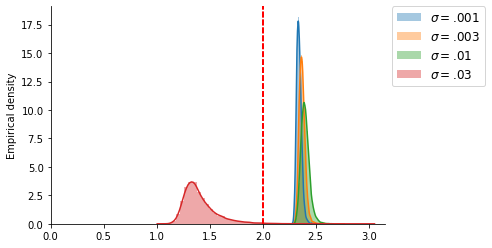}}
    \caption{Scalar-curvature estimation on $S^2$ with isotropic Gaussian noise (standard deviation $\sigma$) added to the point cloud. (A) We use a Gaussian kernel for density estimation. The kernel takes the estimated geodesic distances as input. (B) We use a biweight kernel that takes estimated geodesic distances as input. (C) We use a Gaussian kernel that takes Euclidean distances as input. (D) We use a biweight kernel that takes Euclidean distances as input.}
    \label{fig:noise}
\end{figure}

%%%%

\section{Conclusions}

In this paper, we described a new method to estimate scalar curvature in discrete data. The only information that our approach requires is the set of pairwise distances between the points. By contrast, prior methods were limited to surfaces in $\bR^3$ or to point clouds embedded in Euclidean space. Because our method depends only on metric data, one can use it to estimate curvature not only in point-cloud data (from which geodesic distances can be estimated using the approach in~\cite{isomap, graph_approx}, for example), but also at vertices in a graph that is equipped with the shortest-path metric or at finite samples from an arbitrary metric space (e.g., the Billera-Holmes-Vogtmann space of phylogenetic trees~\cite{BHV}). We proved that under suitable conditions, our estimator is stable (Theorem~\ref{thm:stability}) and that it converges to the ground-truth scalar curvature (Theorem~\ref{thm:convergence}). 

We validated our method on several synthetic data sets in Section \ref{sec:experiments}. Notably, our experiments included a data set (a point cloud that is sampled from the Poincar\'e disk) for which we possessed only the pairwise exact geodesic distances, not an embedding of the points in Euclidean space. Our experiments on point-cloud data embedded in Euclidean space are equivalent to experiments on \emph{geometric graphs}, which are graphs in which vertices are sampled from a manifold, edges connect nearby points, and edge weights are given by distances. This is because we estimated geodesic distance in our point clouds by constructing a nearest-neighbor graph (which is a type of geometric graph) and computing shortest-path lengths. Therefore, our method for scalar-curvature estimation on a point cloud is equivalent to scalar-curvature estimation on the nearest-neighbor graph equipped with the shortest-path metric. Our experiments show that one can achieve reasonable accuracy even without having or using a Euclidean embedding of the data.

The primary limitation of our estimator is that it can be inaccurate on regions with non-constant curvature, especially near points on a manifold where a local extremum in the curvature is attained. (For example, see our experiments on the torus and hyperboloid in Section \ref{sec:experiments}.) The reason is that when the radius $r$ is small, we cannot reliably estimate the ratio between $\volM{x}{r}$ and the volume of a Euclidean ball of radius $r$ (see Proposition~\ref{prop:MSE}). We addressed this by using a relatively high $\rmax$ parameter, which controls the maximum ball radius that we consider. However, requiring $r$ to be relatively large has the drawback that we are unable to detect local variation in scalar curvature; we are effectively smoothing out the curvature. In future work, we plan to investigate strategies to increase the accuracy of our method on manifolds with non-constant scalar curvature.

We expect that our scalar-curvature estimator will improve with improvements in state-of-the-art methods for density and geodesic-distance estimation on manifolds. Our method involves density estimation  on a manifold as an intermediary step, and it also requires geodesic-distance estimation when we are given a point cloud embedded in Euclidean space instead of a distance matrix $d_X$. There are several other methods for geodesic-distance estimation that we did not use in our experiments; see \cite{arvan, spherelets}, for example. Improvements to the intermediary steps of our pipeline will lead to better performance of our scalar-curvature estimator.

It would also be interesting to incorporate machine learning into our curvature-estimation pipeline. For example, at each point, we estimate a sequence of ball-volume ratios (see Eq. \eqref{eq:est_ratio}); this is a vector that one can feed into a neural network, rather than using the method in Section \ref{sec:quad} for estimating a quadratic coefficient. One could also use a graph neural network in which the graph is the nearest-neighbor graph for the data set and the initial node features are the vectors of ball-volume ratio estimates. Using machine learning would allow one to sidestep the choices of hyperparameters (the maximum ball radius $\rmax$, the minimum ball radius $\rmin$, and the radius sequence $\{r_j\}$), although those decisions would be replaced by different hyperparameter choices (e.g., a choice of learning rate). However, our current approach has the advantage that it is highly interpretable. We have designed our method so that, at minimum, one can reliably trust that the scalar curvature sign is accurate---in many cases, the sign of the curvature is the qualitative information that matters most---and that our method will generalize to manifolds that are not present in the training data set.

% \appendix

\section*{Appendix}
\renewcommand{\thesection}{A}
Here we prove some technical lemmas for proving our stability theorem (Theorem \ref{thm:stability}) and convergence theorem (Theorem \ref{thm:convergence}). The notation that we use is the same as in Sections \ref{sec:stability} and \ref{sec:convergence}.

\begin{lemma}\label{lem:small_ball_volume}
If $M$ is compact, then there are positive constants $B^{(1)}$ and $B^{(2)}$ such that 
\begin{equation}\label{eq:ball_bound}
    B^{(1)}r^n \leq \volM{x}{r} \leq B^{(2)}r^n
\end{equation}
for sufficiently small $r$ and all $x$ in $M$.
\end{lemma}
\begin{proof}
    By equation~\eqref{eq:scalar_ball}, there are positive constants $B^{(1)}_x$, $B^{(2)}_x$, and $r_x$  for each $x \in M$ such that
\begin{equation*}
B^{(1)}_xr^n \leq \volM{x}{r} \leq B^{(2)}_xr^n \qquad \text{for } r < r_x\,.
\end{equation*}
Because the Riemannian metric $g$ is smooth, the quantities $r'_x$, $B^{(1)}_x$, and $B^{(2)}_x$ can be chosen for each $x \in M$ such that the functions $x \mapsto r'_x$, $x \mapsto B^{(1)}_x$, and $x \mapsto B^{(2)}_x$ are continuous. If $M$ is compact, then $B^{(i)} := \max_{x \in M} B^{(i)}_x$ and $r_* := \min_{x \in M} r_x$ exist, so Eq.~\eqref{eq:ball_bound} holds for $r < r_*$ and all $x$ in $M$.
\end{proof}

\begin{lemma}\label{lem:var_volhat}
Let $X$ be a point cloud that consists of $N$ points drawn from pdf $\rho: M \to \mathbb{R}_+$. Then
	\begin{equation*}
		\var(\volhat{d_X}{\rho}{x}{r}) = \frac{\var(1/\rho(z)) \cdot \avgrho{x}{r} \cdot \volM{x}{r}}{\Nothers} + \frac{\var N(x, r)}{(\Nothers)^2 \avgrho{x}{r}^2}\,,
	\end{equation*}
	where $z \in \BM{x}{r}$ is a point chosen randomly from the pdf $\psi(z)$ defined in Eq. \eqref{eq:random_ball_point} and
	\begin{equation*}
		\var(N(x, r)) = (\Nothers)\avgrho{x}{r}\volM{x}{r}(1 - \avgrho{x}{r}\volM{x}{r})\,.
	\end{equation*}
\end{lemma}
\begin{proof}
	By Lemma \ref{lem:E_volhat},
	\begin{equation*}
		\var(\volhat{d_X}{\rho}{x}{r}) = \E\Big[(\volhat{d_X}{\rho}{x}{r}- \volM{x}{r})^2\Big]\,.
	\end{equation*}
By equation~\eqref{eq:E_volhat_condition_Nxr},
	\begin{align}
		&\E\Big[(\volhat{d_X}{\rho}{x}{r} - \volM{x}{r})^2 \Big \vert N(x, r) = k \Big] \notag \\
		&\qquad = \E[\volhat{d_X}{\rho}{x}{r}^2 \mid N(x, r) = k] 
		- \frac{2 k\volM{x}{r}}{(\Nothers)\avgrho{x}{r}} + \volM{x}{r}^2\label{eq:volhat_MSE}
	\end{align}
for all $k \in \{0, \ldots, \Nothers\}$. By equation~\eqref{eq:volest_alt},
	\begin{equation}\label{eq:volhatsquared}
		\E[\volhat{d_X}{\rho}{x}{r}^2 \mid N(x, r) = k]  = \frac{k^2}{(\Nothers)^2} \cdot \E\Big[\Big(\fracc{k} \sum_{i=1}^k 1/\rho(z_i) \Big)^2\Big]\,,
	\end{equation}
	where $\{z_i\}_{i=1}^k = \BM{x}{r} \cap (X \setminus \{x\})$. If $k \geq 1$, the quantity $\fracc{k} \sum_{i=1}^k 1/\rho(z_i)$ is a sample mean. Therefore,
	\begin{equation}\label{eq:samplemeansquared}
		\E\Big[\Big(\fracc{k} \sum_{i=1}^k 1/\rho(z_i) \Big)^2\Big] = \frac{\var(1/\rho(z))}{k} + \E[1/\rho(z)]^2 = \frac{\var(1/\rho(z))}{k} + \fracc{\avgrho{x}{r}^2}\,,
	\end{equation}
	where $z$ is chosen from the pdf $\psi(z)$ defined in equation~\eqref{eq:random_ball_point}. The last equality follows by equation~\eqref{eq:E_1/rho}. Substituting equation~\eqref{eq:samplemeansquared} into equation~\eqref{eq:volhatsquared} and equation~\eqref{eq:volhatsquared} into~\eqref{eq:volhat_MSE}, we obtain
	\begin{align}
		&\E\Big[(\volhat{d_X}{\rho}{x}{r} - \volM{x}{r})^2 \Big \vert N(x, r) = k \Big] \notag \\
		&= \Bigg[ \frac{k}{(\Nothers)^2} \cdot \var\Big(\fracc{\rho(z)}\Big) \notag \\
        &\qquad + \frac{k^2 - 2k(\Nothers) \avgrho{x}{r} \volM{x}{r} + (\Nothers)^2 \avgrho{x}{r}^2 \volM{x}{r}^2}{(\Nothers)^2 \avgrho{x}{r}^2} \Bigg]\label{eq:volhat_var}
	\end{align}
	for $k \in \{1, \ldots, \Nothers\}$. When $k = 0$, equation~\eqref{eq:volhat_var} holds because the righthand side equals $\volM{x}{r}$ and $\E[(\volhat{d_X}{\rho}{x}{r} - \volM{x}{r})^2 \mid N(x, r) = 0] = \volM{x}{r}$.
 
 To simplify the righthand side of equation~\eqref{eq:volhat_var}, we observe that $N(x, r)$ is a binomial random variable with $\Nothers$ trials and success probability $\avgrho{x}{r}\volM{x}{r}$, so
	\begin{equation*}
		\E[N(x, r)] = (\Nothers) \avgrho{x}{r}\volM{x}{r}
	\end{equation*}
	and
	\begin{align*}
		&\E[(N(x, r) - \E N(x, r))^2 \mid N(x, r) = k] \\
		&\qquad = k^2 - 2k(\Nothers)\avgrho{x}{r}\volM{x}{r} + (\Nothers)^2 \avgrho{x}{r}^2\volM{x}{r}\,.
	\end{align*}
	Therefore
	\begin{align*}
		&\E\Big[(\volhat{d_X}{\rho}{x}{r} - \volM{x}{r})^2 \Big \vert N(x, r) = k \Big] \\
		&\qquad = \frac{k}{(\Nothers)^2} \cdot \var(1/\rho(z))
		+ \fracc{(\Nothers)^2 \avgrho{x}{r}^2} \cdot \E[(N(x, r) - \E N(x, r))^2 \mid N(x, r) = k]\,.
	\end{align*}
	Putting it all together, we have that $\var(\volhat{d_X}{\rho}{x}{r})$ equals
	\begin{align*}
		&\E\Big[(\volhat{d_X}{\rho}{x}{r} - \volM{x}{r})^2\Big] \\
		&\qquad = \sum_{k=0}^{\Nothers} \E\Big[(\volhat{d_X}{\rho}{x}{r}- \volM{x}{r})^2 \Big\vert N(x, r) = k\Big]\p[N(x, r) = k]\\
		&\qquad = \Bigg( \frac{\var(1/\rho(z))}{(\Nothers)^2}\sum_{k=0}^{\Nothers} k \cdot \p[N(x, r) = k] \\
		&\qquad\qquad + \fracc{(\Nothers)^2 \avgrho{x}{r}^2} \sum_{k=0}^{\Nothers} \E[(N(x, r) - \E N(x, r))^2 \mid N(x, r) = k] \p[N(x, r) = k] \Bigg)\\
		&\qquad = \frac{\var(1/\rho(z))}{(\Nothers)^2}\E[N(x, r)] + \frac{\var(N(x, r))}{(\Nothers)^2 \avgrho{x}{r}^2} \\
		&\qquad = \frac{\var(1/\rho(z)) \cdot \avgrho{x}{r} \cdot \volM{x}{r}}{(\Nothers)} + \frac{\var (N(x, r))}{(\Nothers)^2 \avgrho{x}{r}^2}\,,
	\end{align*}
where
 \begin{equation*}
     \var(N(x, r)) = (\Nothers)\avgrho{x}{r}\volM{x}{r}(1 - \avgrho{x}{r}\volM{x}{r})
 \end{equation*}
 because $N(x, r)$ is a binomial random variable with parameters $\Nothers$ and $\avgrho{x}{r} \volM{x}{r}$.
\end{proof}

\begin{lemma}\label{lem:yhatminus1}
Assume that $\hat{n}[\dkexact] = n$ for sufficiently large $k$. Let $\{a_k\}_{k=1}^{\infty}$ and $\{b_k\}_{k=1}^{\infty}$ be sequences of positive real numbers such that 
\begin{enumerate}
    \item $ 0< a_k < b_k$ for all $k$, and
    \item $a_k, b_k \to 0$ as $k \to \infty$\,.
\end{enumerate}
For each $k$, let $R_k$ be a finite subset of $[a_k, b_k]$ such that $\frac{|R_k|}{|X_k|a_k^n} \to 0$ as $k \to 0$. Then
\begin{equation*}
   \max_{r \in R_k} \vert \yhat{\dkexact}{\rho}{x_k}{r} - 1 \vert \to 0
\end{equation*}
in probability as $k \to \infty$, where $\{x_k\}$ is any sequence of points such that $x_k \in X_k$.
\end{lemma}
\begin{proof}
Let $\epsilon > 0$. To simplify our notation, we denote $\yhat{\dkexact}{\rho}{x}{r}$ by $\hat{y}(x, r)$. For any $x \in X_k$ and any $r \in [a_k, b_k]$,
\begin{equation}\label{eq:yminus1}
    \p\Big[|\hat{y}(x,r) - 1| > \epsilon\Big] \leq \p\Big[ |\hat{y}(x, r) - y(x, r)| + |y(x,r) - 1|> \epsilon\Big]\,.
\end{equation}
By equation~\eqref{eq:scalar_ball}, there are constants $A > 0$ and $r_1 > 0$ such that
\begin{equation*}
    |y(x, r) - 1| \leq Ar^2 \qquad \text{for } r < r_1 \text{ and all } x\in M\,.
\end{equation*}
Let $r_2 = \min(\sqrt{\epsilon/(2A)}, r_1)$. If $r < r_2$, then $|y(x,r) - 1| < \frac{\epsilon}{2}$. For sufficiently large $k$, we have $b_k < r_2$, so by equation~\eqref{eq:yminus1},
\begin{equation}\label{eq:lemma_reduction}
    \p\Big[|\hat{y}(x,r) - 1| > \epsilon\Big] \leq \p\Big[|\hat{y}(x, r) - y(x, r)| > \epsilon/2\Big]
\end{equation}
for any $r \in [a_k, b_k]$ and for sufficiently large $k$. By Chebyshev's inequality,
\begin{equation}\label{eq:cheby_yhat}
    \p\Big[|\hat{y}(x, r) - y(x, r)| > \epsilon/2\Big] \leq \frac{4 \var(\hat{y}(x,r))}{\epsilon^2}\,.
\end{equation}
By Proposition \ref{prop:MSE}, there are positive constants $B$ and $r_3 < r_2$ such that
\begin{equation*}
    \var(\hat{y}(x,r)) \leq \frac{B}{|X_k|r^n}
\end{equation*}
for sufficiently large $k$, all $r < r_3$, and any $x \in X_k$. Substituting into equation~\eqref{eq:cheby_yhat} shows that
\begin{equation*}
     \p\Big[|\hat{y}(x, r) - y(x, r)| > \epsilon/2\Big] \leq \frac{4B}{\epsilon^2 |X_k|r^n}
\end{equation*}
for any $r < r_3$ and any $x\in X_k$. For sufficiently large $k$, we have $b_k < r_3$, so
\begin{equation*}
    \p\Big[|\hat{y}(x, r) - y(x, r)| > \epsilon/2\Big] \leq \frac{4B}{\epsilon^2|X_k|a_k^n}
\end{equation*}
for any $r \in [a_k, b_k]$ and for sufficiently large $k$. Therefore,
\begin{equation*}
    \p\Big[\max_{r \in R_k} |\hat{y}(x, r) - y(x, r)| > \epsilon/2\Big] \leq \frac{4B|R_k|}{\epsilon^2 |X_k|a_k^n}
\end{equation*}
By hypothesis, the righthand side approaches $0$ as $k \to \infty$ because $\frac{|R_k|}{|X_k|a_k^n} \to 0$. Applying equation~\eqref{eq:lemma_reduction} concludes the proof.
\end{proof}

\begin{lemma}[Stability of $\hat{y}$]\label{lem:yhat_stability}
For each $k$, suppose that $\dkapprox$ is a metric on $X_k$ such that
\begin{equation*}
	\delta_k := \max_{x, x' \in X_k} |\dkapprox(x, x') - d(x, x')| \to 0 \qquad \text{as } k \to \infty
\end{equation*}
and $\hat{\rho}$ is a density estimator such that
\begin{equation*}
		\eta_k := \max_{x \in X_k} |\hat{\rho}[\dkapprox](x) - \rho(x)| \to 0 \qquad \text{as } k \to \infty\,.
\end{equation*}
Suppose that $\hat{n}[\dkapprox] = \hat{n}[\dkexact] = n$ for sufficiently large $k$. Additionally, suppose that the hyperparameter value sequences satisfy the conditions:
\begin{enumerate}
	\item $(\rmink + \Drk)/\rmaxk^3 \to 0$ as $k \to \infty$\,.
	\item $\eta_k/(\rmink + \Drk)^{n +2/3}\to 0$ as $k \to \infty$\,.
    \item $\max_j \frac{A(r_j)}{r_j^n \rmaxk^2} \to 0$ as $k \to \infty$, where $A(r)$ is defined as in equation~\eqref{eq:Adef}.
    % \item $|X_k|\Drk (\rmink + \Drk)^n \to \infty$ as $k \to \infty$\,.
	\item $\rmink + \Drk - \delta_k > 0$ for sufficiently large $k$\,.
	\item $\frac{\delta_k + \Drk}{(\rmink + \Drk - \delta_k)^{n+1} \rmaxk^2}\to 0$ as $k \to \infty$\,.
\end{enumerate}
Define $\ell_k := \min\{\ell \in \mathbb{Z} \mid \ell \Drk \geq \delta_k\}$. Then there is a sequence $\{\xi_k\}$ of nonnegative real numbers satisfying $\xi_k/\rmaxk^2 \to 0$ as $k \to \infty$ such that for any sequence $\{x_k\}_{k=1}^{\infty}$, where $x_k \in X_k$ for all $k$,
\begin{equation*}
    \yhat{\dkexact}{\rho}{x_k}{r_{j-\ell_k}} - \xi_k \leq \yhat{\dkapprox}{\hat{\rho}}{x_k}{r_j} \leq \yhat{\dkexact}{\rho}{x_k}{r_{j+\ell_k}} + \xi_k
\end{equation*}
for all $j \geq 2$ and
\begin{equation*}
    \yhat{\dkexact}{\rho}{x_k}{r_1 - \delta_k} - \xi_k \leq \yhat{\dkapprox}{\hat{\rho}}{x_k}{r_1} \leq \yhat{\dkexact}{\rho}{x_k}{r_{1+\ell_k}} + \xi_k
\end{equation*}
for $j = 1$.
\end{lemma}

\begin{proof}
For convenience, let $\hat{y}_k(x, r)$ denote $\yhat{\dkapprox}{\hat{\rho}}{x}{r}$ and let $\hat{y}(x, r)$ denote $\yhat{\dkexact}{\rho}{x}{r}$. Define
	\begin{align*}
		\lambda_{j, k}^+ &:= \ell_k \Drk\,,\\
		\lambda_{j, k}^- &:= \begin{cases}
			\ell_k \Drk\,, & j \geq 2 \\
			\delta_k\,, & j = 1
		\end{cases}
	\end{align*}
	for all $j$ and $k$. Our goal is to compare $\hat{y}_k(x, r_j)$ to both $\hat{y}(x, r_j - \lambda_{j, k}^-)$ and $\hat{y}(x, r_j + \lambda_{j, k}^+)$. The ``radial shift values'' $\lambda_{j, k}^\pm$ are defined so that they satisfy 
	\begin{enumerate}
	\item $r_j \pm \lambda_{j, k}^\pm > 0$ (all shifted radius values are positive) and
 	\item $\delta_k \leq \lambda_{j, k}^{\pm} \leq \delta_k + \Drk$
	\end{enumerate}
	for all $j$ and sufficiently large $k$. For the remainder of the proof, we consider only $k$ sufficiently large such that (1) holds. The key is that because $\lambda_{j, k}^\pm \geq \delta_k$ for all $j$ and $k$,
        \begin{equation}\label{eq:Nr_bound}
		\Ndxr{\dkexact}{x}{r - \lambda_{j, k}^-} \leq \Ndxr{\dkapprox}{x}{r} \leq \Ndxr{\dkexact}{x}{r + \lambda_{j, k}^+}
	\end{equation}
	for all $x \in X_k$ and $r \geq 0$. We use Eq.~\eqref{eq:Nr_bound} to compare $\volhat{\dkapprox}{\hat{\rho}}{x}{r}$ and $\volhat{\dkexact}{\rho}{x}{r \pm \lambda_{j, k}^{\pm}}$. First we quantify the error introduced by the error in density estimation. Observe that
	\begin{align}
		\Big\vert \volhat{\dkapprox}{\hat{\rho}}{x}{r} - \volhat{\dkapprox}{\rho}{x}{r} \big\vert &=  \Big\vert  \frac{\Ndxr{\dkapprox}{x}{r}}{(|X_k| - 1)\avgrhoest{\hat{\rho}}{x}{r}} - \frac{\Ndxr{\dkapprox}{x}{r}}{(|X_k| - 1)\avgrhoest{\rho}{x}{r}} \Big\vert \notag \\
		&\leq \Big\vert \frac{1}{ \avgrhoest{\hat{\rho}}{x}{r}} - \fracc{\avgrhoest{\rho}{x}{r}} \Big\vert \notag\,. \\
  &\leq \Big\vert \frac{1}{ \avgrhoest{\hat{\rho}}{x}{r}} - \fracc{\rho(x)} \Big\vert + \Big\vert \fracc{\avgrhoest{\rho}{x}{r}} - \fracc{\rho(x)} \Big\vert \label{eq:error_due_density}
\end{align}
If $\Ndxr{\dkapprox}{x}{r} \geq 1$ and $d_{X_k}(x, z) \leq r$, then
\begin{equation*}
    \Big\vert \fracc{\rho(z)} - \fracc{\rho(x)} \Big\vert \leq \frac{A(r)}{\min(\rho)^2}
\end{equation*}
and
\begin{align*}
    \Big\vert \fracc{\hat{\rho}(z)} - \fracc{\rho(x)}\Big\vert &\leq \Big\vert \fracc{\hat{\rho}(z)} - \fracc{\rho(z)}\Big\vert + \Big\vert \fracc{\rho(z)} - \fracc{\rho(x)}\Big\vert \\
    &\leq \frac{\eta_k}{(\min(\rho) - \eta_k)^2} + \frac{A(r)}{\min(\rho)^2}\,.
\end{align*}
If $\Ndxr{\dkapprox}{x}{r} = 0$, then
\begin{align*}
    \Big\vert \frac{1}{ \avgrhoest{\hat{\rho}}{x}{r}} - \fracc{\rho(x)} \Big\vert &= \Big\vert \fracc{\hat{\rho}(x)} - \fracc{\rho(x)} \Big\vert \leq \frac{\eta_k}{(\min(\rho) - \eta_k)^2}\,,\\
     \Big\vert \fracc{\avgrhoest{\rho}{x}{r}} - \fracc{\rho(x)} \Big\vert &= 0\,.
\end{align*}
Therefore,
\begin{equation*}
    \Big\vert \frac{1}{ \avgrhoest{\hat{\rho}}{x}{r}} - \fracc{\rho(x)} \Big\vert + \Big\vert \fracc{\avgrhoest{\rho}{x}{r}} - \fracc{\rho(x)} \Big\vert \leq \frac{\eta_k}{(\min(\rho) - \eta_k)^2} + \frac{2A(r)}{\min(\rho)^2}\,.
\end{equation*}
By Eq. \eqref{eq:error_due_density},
\begin{align}
		&\volhat{\dkapprox}{\rho}{x}{r_j} - \Big(\frac{\eta_k}{(\min(\rho) - \eta_k)^2} + \frac{2A(r_j)}{\min(\rho)^2}\Big) \leq \volhat{\dkapprox}{\hat{\rho}}{x}{r_j} \notag \\
  &\qquad \leq \volhat{\dkapprox}{\rho}{x}{r_j} + \frac{\eta_k}{(\min(\rho) - \eta_k)^2} + \frac{2A(r_j)}{\min(\rho)^2}\,.\label{eq:volhat_intermediate}
\end{align}
	Together, equations~\eqref{eq:Nr_bound} and~\eqref{eq:volhat_intermediate} show that
\begin{align}
	\frac{\Ndxr{\dkexact}{x}{r_j - \lambda_{j, k}^-}}{(|X_k| - 1) \avgrhoest{\rho}{x}{r_j}} - \Big(\frac{\eta_k}{(\min(\rho) - \eta_k)^2} + \frac{2A(r_j)}{\min(\rho)^2} \Big) \leq \volhat{\dkapprox}{\hat{\rho}}{x}{r_j} \notag \\
 \leq \frac{\Ndxr{\dkexact}{x}{r_j + \lambda_{j, k}^+}}{(|X_k| - 1) \avgrhoest{\rho}{x}{r_j}} + \Big(\frac{\eta_k}{(\min(\rho) - \eta_k)^2} + \frac{2A(r_j)}{\min(\rho)^2}\Big)\,.\label{eq:volhat_comparison}
\end{align}
We have
\begin{align*}
    &\Big\vert \frac{\Ndxr{\dkexact}{x}{r_j \pm\lambda_{j, k}^\pm}}{(|X_k| - 1) \avgrhoest{\rho}{x}{r_j}} - \volhat{\dkexact}{\rho}{x}{r_j \pm \lambda_{j, k}^{\pm}} \Big\vert \\
    &\qquad \leq \Big\vert \fracc{\avgrhoest{\rho}{x}{r_j}} - \fracc{\avgrhoest{\rho}{x}{r_j \pm \lambda_{j, k}^{\pm}}} \Big\vert \\
    &\qquad \leq \Big\vert \fracc{\avgrhoest{\rho}{x}{r_j}} - \fracc{\rho(x)} \Big\vert + \Big\vert \fracc{\avgrhoest{\rho}{x}{r_j \pm \lambda_{j, k}^{\pm}}}  - \fracc{\rho(x)} \Big\vert \\
    &\qquad \leq \frac{A(r_j)}{\min(\rho)^2} + \frac{A(r_j + \lambda_{j, k}^+)}{\min(\rho)^2} \\
    &\qquad \leq \frac{2A(r_j + \lambda_{j, k}^+)}{\min(\rho)^2}\,.
\end{align*}
Therefore, by Eq. \eqref{eq:volhat_comparison},
\begin{align*}
\volhat{\dkexact}{\rho}{x}{r_j - \lambda_{j, k}^{-}}  - \Big(\frac{\eta_k}{(\min(\rho) - \eta_k)^2} + \frac{2A(r_j)}{\min(\rho)^2} + \frac{2A(r_j + \lambda_{j, k}^+)}{\min(\rho)^2}\Big) \leq \volhat{\dkapprox}{\hat{\rho}}{x}{r_j} \notag \\
 \leq \volhat{\dkexact}{\rho}{x}{r_j + \lambda_{j, k}^{+}}  + \Big(\frac{\eta_k}{(\min(\rho) - \eta_k)^2} + \frac{2A(r_j)}{\min(\rho)^2}+ \frac{2A(r_j + \lambda_{j, k}^+)}{\min(\rho)^2}\Big)\,.
\end{align*}
Because $A(r)$ increases monotonically,
\begin{align}
    \volhat{\dkexact}{\rho}{x}{r_j - \lambda_{j, k}^{-}}  - \Big(\frac{\eta_k}{(\min(\rho) - \eta_k)^2}  + \frac{4A(r_j + \lambda_{j, k}^+)}{\min(\rho)^2}\Big) \leq \volhat{\dkapprox}{\hat{\rho}}{x}{r_j} \notag \\
 \leq \volhat{\dkexact}{\rho}{x}{r_j + \lambda_{j, k}^{+}}  + \Big(\frac{\eta_k}{(\min(\rho) - \eta_k)^2} + \frac{4A(r_j + \lambda_{j, k}^+)}{\min(\rho)^2}\Big)\,.\label{eq:volhat_comparison2}
\end{align}
Next, we use equation~\eqref{eq:volhat_comparison2} to compare $\hat{y}_k(x, r_j)$ to $\hat{y}(x, r_j \pm \lambda_{j, k}^\pm)$ for all $j \in \{1, \ldots, m_k\}$. Dividing equation~\eqref{eq:volhat_comparison2} by $v_nr_j^n$, we obtain
\begin{align}
	\frac{\volhat{\dkexact}{\rho}{x}{r_j - \lambda_{j, k}^-}}{v_n r_j^n} - \Big(\frac{\eta_k}{(\min(\rho) - \eta_k)^2v_nr_j^n} + \frac{4A(r_j + \lambda_{j, k}^+)}{v_nr_j^n\min(\rho)^2} \Big) \leq \hat{y}_k(x, r_j) \notag \\
 \leq \frac{\volhat{\dkexact}{\rho}{x}{r_j + \lambda_{j, k}^+}}{v_nr_j^n} + \Big(\frac{\eta_k}{(\min(\rho) - \eta_k)^2v_nr_j^n} + \frac{4A(r_j + \lambda_{j, k}^+)}{v_nr_j^n\min(\rho)^2}\Big)\label{eq:yhat_intermediate}
\end{align}
for all $j$. We now compare $\volhat{\dkexact}{\rho}{x}{r_j \pm \lambda_{j, k}^\pm}/(v_n r_j^n)$ to $\hat{y}(x, r_j \pm \lambda_{j, k}^\pm)$. We have
\begin{align*}
	\Big \vert \hat{y}(x, r_j \pm \lambda_{j, k}^\pm) - \frac{\volhat{\dkexact}{\rho}{x}{r_j \pm \lambda_{j, k}^\pm}}{v_n r_j^n} \Big \vert 
	&= \frac{\volhat{\dkexact}{\rho}{x}{r_j \pm \lambda_{j, k}^\pm}}{v_n} \Bigg\vert \fracc{(r_j \pm \lambda_{j, k}^\pm)^n} - \fracc{r^n} \Bigg\vert \\
    &= \frac{\Ndxr{\dkexact}{x}{r_j \pm \lambda_{j, k}^\pm}}{ (|X_k| -1)v_n \avgrhoest{\rho}{x}{r}}\Bigg\vert \fracc{(r_j \pm \lambda_{j, k}^\pm)^n} - \fracc{r_j^n} \Bigg\vert \\
	&\leq \fracc{v_n \avgrhoest{\rho}{x}{r}}\Big\vert \fracc{(r_j \pm \lambda_{j, k}^\pm)^n} - \fracc{r_j^n} \Big\vert \\
 &\leq \fracc{v_n \min(\rho)}\Big\vert \fracc{(r_j \pm \lambda_{j, k}^\pm)^n} - \fracc{r_j^n} \Big\vert\,.
\end{align*}
Because $g(r) = 1/r^n$ is convex and monotonically decreasing for $r > 0$, we have
\begin{equation*}
	\Big\vert \fracc{(r_j+\lambda_{j, k}^+)^n} - \fracc{r_j^n} \Big\vert \leq \lambda_{j, k}^+ \vert g'(r_j) \vert = \frac{n \lambda_{j, k}^+}{r_j^{n+1}}
\end{equation*}
and
\begin{equation*}
	\Big\vert \fracc{(r_j-\lambda_{j, k}^-)^n} - \fracc{r_j^n} \Big\vert \leq \lambda_{j, k}^- \vert g'(r_j - \lambda_{j, k}^-) \vert = \frac{n \lambda_{j, k}^-}{(r_j - \lambda_{j, k}^-)^{n+1}}\,.
\end{equation*}
Therefore,
\begin{equation}\label{eq:yhat_approx_+}
	\Big\vert \hat{y}(x, r_j + \lambda_{j, k}^+) - \frac{\volhat{\dkexact}{\rho}{x}{r_j + \lambda_{j, k}^+}}{v_n r_j^n} \Big \vert \leq \fracc{v_n \min(\rho)} \frac{n \lambda_{j, k}^+}{r_j^{n+1}}
\end{equation}
and
\begin{equation}\label{eq:yhat_approx_-}
	\Big\vert \hat{y}(x, r_j - \lambda_{j, k}^-) - \frac{\volhat{\dkexact}{\rho}{x}{r_j - \lambda_{j, k}^-}}{v_n r_j^n} \Big \vert \leq \fracc{v_n \min(\rho)} \frac{n \lambda_{j, k}^-}{(r_j - \lambda_{j, k}^-)^{n+1}}\,.
\end{equation}
Together, Equations \eqref{eq:yhat_intermediate}, \eqref{eq:yhat_approx_+}, and \eqref{eq:yhat_approx_-} show that
\begin{align*}
	\hat{y}_k(x, r_j) &\geq \hat{y}(x, r_j - \lambda_{j, k}^-) - \Bigg(\frac{\eta_k}{(\min(\rho) - \eta_k)^2v_nr_j^n} +  \frac{n \lambda_{j, k}^-}{v_n \min(\rho)(r_j - \lambda_{j, k}^-)^{n+1}} + \frac{4A(r_j + \lambda_{j, k}^+)}{v_nr_j^n \min(\rho)^2}\Bigg)\,, \\
	\hat{y}_k(x, r_j) &\leq \hat{y}(x, r_j + \lambda_{j, k}^+) + \Bigg( \frac{\eta_k}{(\min(\rho) - \eta_k)^2v_nr_j^n} + \frac{n \lambda_{j, k}^+}{v_n \min(\rho) r_j^{n+1}} + \frac{4A(r_j + \lambda_{j, k}^+)}{v_nr_j^n \min(\rho)^2}\Bigg)\,.
\end{align*}
We define the following error terms:
\begin{align*}
	\xi_{j, k}^+ &:= \frac{\eta_k}{(\min(\rho) - \eta_k)^2v_nr_j^n} + \frac{n \lambda_{j, k}^+}{v_n \min(\rho) r_j^{n+1}} + \frac{4A(r_j + \lambda_{j, k}^+)}{v_nr_j^n \min(\rho)^2}\,, \\
	\xi_{j, k}^- &:= \frac{\eta_k}{(\min(\rho) - \eta_k)^2v_nr_j^n}  + \frac{n \lambda_{j, k}^-}{v_n \min(\rho)(r_j - \lambda_{j, k}^-)^{n+1}} + \frac{4A(r_j + \lambda_{j, k}^+)}{v_nr_j^n \min(\rho)^2}\,.\\
	\xi_k &:= \max_j \{\xi_{j, k}^+, \xi_{j, k}^-\}\,.
\end{align*}
The error terms $\xi_{j,k}^{\pm}$ are nonnegative. To complete the proof, it suffices to show that $\xi_k/\rmaxk^2 \to 0$ as $k \to \infty$. For sufficiently large $k$,
\begin{equation*}
\frac{\eta_k}{v_n r_j^n(\min(\rho) - \eta_k)^2 \rmaxk^2} 
% \leq \frac{\eta_k}{v_n r_j^n(\frac{1}{2}\rho)^2 \rmaxk^2} 
\leq \frac{\eta_k}{v_n (\rmink + \Drk)^n(\frac{1}{2}\min(\rho))^2 \rmaxk^2}\,.
\end{equation*}
Rearranging the terms on the righthand side, we obtain
\begin{equation*}
\frac{\eta_k}{v_n (\rmink + \Drk)^n(\frac{1}{2}\rho)^2 \rmaxk^2} = \frac{4}{v_n \min(\rho)^2} \Bigg(\frac{\eta_k}{(\rmink + \Drk)^{n + 2/3}} \Bigg) \Bigg(\frac{(\rmink + \Drk)^{2/3}}{\rmaxk^2}\Bigg)\,.
\end{equation*}
By hypothesis, the quantity above approaches $0$ as $k\to \infty$, so $\max_j\frac{\eta_k}{v_n r_j^n(\min(\rho) - \eta_k)^2 \rmaxk^2} \to 0$ as $k \to \infty$. Additionally,
\begin{align*}
\frac{n\lambda_{j, k}^+}{v_n \min(\rho) r_j^{n+1}\rmaxk^2} 
% &\leq \frac{n(\delta_k + \Drk)}{v_n \rho (\rmink + \Drk)^{n+1} \rmaxk^2} 
&\leq \frac{n(\delta_k + \Drk)}{v_n \min(\rho) (\rmink + \Drk - \delta_k)^{n+1} \rmaxk^2}\,,\\
\frac{n \lambda_{j, k}^-}{v_n \min(\rho)(r_j - \lambda_{j, k}^-)^{n+1}\rmaxk^2} &\leq \frac{n(\delta_k + \Drk)}{v_n \min(\rho) (\rmink + \Drk - \delta_k)^{n+1} \rmaxk^2}\,.
\end{align*}
By hypothesis, the righthand sides above approach $0$ as $k \to \infty$. Finally, we upper bound $\frac{A(r_j + \lambda_{j, k}^+)}{r_j^n \rmaxk^2}$ by recalling that $A(r)$ increases monotonically and
\begin{align*}
    r_j + \lambda_{j, k}^+ &= r_j + \ell_k \Drk \\
    &\leq r_j + \delta_k + \Drk \\
    &\leq 2 r_j
\end{align*}
for sufficiently large $k$. Therefore, $\frac{A(r_j + \lambda_{j, k}^+)}{r_j^n \rmaxk^2} \leq \frac{A(2r_j)}{\rmaxk^2 r_j^n}$, which approaches $0$ by hypothesis. This implies that $\xi_k/\rmaxk^2 \to 0$ as $k \to \infty$.
\end{proof}

\begin{lemma}[Stability of $\hat{C}$]\label{lem:Cstability}
    For each $k$, suppose that $\dkapprox$ is a metric on $X_k$ such that
\begin{equation*}
    \delta_k := \max_{x, x' \in X_k} \vert \dkapprox(x, x') - d(x, x') \vert \to 0 \qquad \text{as } k \to \infty\,.
\end{equation*}
Suppose that $\hat{\rho}$ is a density estimator such that
\begin{equation*}
    \eta_k := \max_{x \in X_k} \Big\vert \hat{\rho}[\dkapprox](x) - \rho(x) \Big\vert \to 0 \qquad \text{as } k \to \infty\,,
\end{equation*}
and suppose that $\hat{n}[\dkapprox] = \hat{n}[\dkexact] = n$ for sufficiently large $k$.
If the hyperparameter value sequences satisfy
\begin{enumerate}
    \item $\max_j \frac{A(2r_j)}{r_j^n \rmaxk^2} \to 0$ as $k \to \infty$\,,
    \item $\eta_k/(\rmink + \Drk)^{n+2/3} \to 0$ as $k \to \infty$\,,
    \item $\rmink + \Drk > \delta_k$ for sufficiently large $k$\,,
    \item $|X_k|\Drk(\rmink + \Drk - \delta_k)^n \to \infty$ as $k \to \infty$\,,
    \item $\rmink/\rmaxk^3 \to 0$ as $k \to \infty$\,,
    \item $(\Drk + \delta_k)/\rmaxk^3 \to 0$ as $k \to \infty$, and
    \item $(\Drk + \delta_k)/[(\rmink + \Drk - \delta_k)^{n+1}\rmaxk^2] \to 0$ as $k \to \infty$
\end{enumerate}
then $\vert \hat{C}[X_k, \dkapprox, \hat{\rho}](x_k) -  \hat{C}[X_k, d, \rho](x_k) \vert \to 0$ in probability as $k \to \infty$, where $\{x_k\}$ is any sequence of points such that $x_k \in X_k$.
\end{lemma}
\begin{proof}
To simplify our notation, we define
\begin{align*}
    \hat{C}_k(x) &:= \hat{C}[\dkapprox, \hat{\rho}](x)\,, \\
    \hat{C}(x) &:= \hat{C}[d, \rho](x)\,, \\
    \hat{y}_k(x, r) &:= \hat{y}[\dkapprox, \hat{\rho}](x, r)\,, \\
    \hat{y}(x, r) &:= \hat{y}[d, \rho](x, r)
\end{align*}
for all $x \in X_k$. Let $\ell_k = \min\{\ell_k' \in \mathbb{Z} \mid \ell_k' \Drk \geq \delta_k \}$, let $a_k = \rmink + \Drk - \delta_k$, and let $b_k = \rmaxk + \ell_k \Drk$. By hypothesis and choice of $\ell_k$,
\begin{align*}
    &a_k > 0 \qquad &\text{for all } k\,, \\
    &|X_k|(a_k)^n \to \infty \qquad &\text{as } k \to \infty\,, \\
    &a_k < (\rmink + \Drk) \to 0 \qquad &\text{as } k \to \infty\,, \\
    &b_k = \rmaxk + \Drk + (\ell_k - 1)\Drk < (\rmaxk+ \Drk + \delta_k) \to 0 \qquad &\text{as } k \to \infty\,.
\end{align*}
Let $J := \{2 - \ell_k, \ldots, m_k + \ell_k\}$, where $m_k := \frac{\rmaxk - \rmink}{\Drk}$ is the number of radial steps. Let $R_k := \{r_j \mid j \in J\} \cup \{\rmink + \Drk + \delta_k\}$. We have
\begin{equation*}
    |R_k| = m_k + 2 \ell_k \leq m_k + \frac{\delta_k}{\Drk} + 1 \leq \frac{2}{\Drk} + 1\,.
\end{equation*}
Because $|X_k|\Drk (\rmink + \Drk - \delta_k)^n \to \infty$, we have $\frac{|R_k|}{|X_k|a_k^n} \to 0$ as $k \to \infty$. Therefore, by Lemma \ref{lem:yhatminus1},
\begin{equation}\label{eq:yhat_near_1}
    \p\Big[\max_{r \in R_k}\vert \hat{y}(x, r) - 1 \vert \leq 1\Big] \to 1
\end{equation}
as $k \to \infty$.

By Lemma \ref{lem:yhat_stability}, there is a nonnegative sequence $\{\xi_k\}$ such that $\xi_k/\rmaxk^2 \to 0$ and
\begin{equation}\label{eq:j>=2_assumption}
    \yhat{\dkexact}{\rho}{x_k}{r_{j-\ell_k}} - \xi_k \leq \yhat{\dkapprox}{\hat{\rho}}{x_k}{r_j} \leq \yhat{\dkexact}{\rho}{x_k}{r_{j+\ell_k}} + \xi_k
\end{equation}
for all $j \geq 2$ and
\begin{equation}\label{eq:j=1_assumption}
\yhat{\dkexact}{\rho}{x_k}{r_1 - \delta_k} - \xi_k \leq \yhat{\dkapprox}{\hat{\rho}}{x_k}{r_1} \leq \yhat{\dkexact}{\rho}{x_k}{r_{1+\ell_k}} + \xi_k
\end{equation}
for sufficiently large $k$. (The case $j = 1$ is different because it is not necessarily true that $r_{1 - \ell_k} \geq 0$.)

Let $\epsilon > 0$.
% For all $x \in X_k$, we observe that
% \begin{align}
%     &\p[ \vert \hat{C}_k(x) - \hat{C}(x) \vert > \epsilon] \notag \\
%     &\qquad = \Bigg[ \Bigg(\p\Big[ \vert \hat{C}_k(x) - \hat{C}(x) \vert > \epsilon \Big\vert \sup_{r \in [a_k, b_k]}\vert \hat{y}(x, r) - 1 \vert \leq 1 \Big] \cdot \p\Big[\sup_{r \in [a_k, b_k]}\vert \hat{y}(x, r) - 1 \vert \leq 1 \Big] \Bigg) \notag \\
%     &\qquad\qquad  + \Bigg(\p\Big[ \vert \hat{C}_k(x) - \hat{C}(x) \vert > \epsilon \Big\vert \sup_{r \in [a_k, b_k]}\vert \hat{y}(x, r) - 1 \vert > 1 \Big] \cdot \p\Big[\sup_{r \in [a_k, b_k]}\vert \hat{y}(x, r) - 1 \vert > 1 \Big]\Bigg) \Bigg] \notag \\
%     &\qquad \leq \p\Big[ \vert \hat{C}_k(x) - \hat{C}(x) \vert > \epsilon \Big\vert \sup_{r \in [a_k, b_k]}\vert \hat{y}(x, r) - 1 \vert \leq 1 \Big] + \p\Big[\sup_{r \in [a_k, b_k]} \vert \hat{y}(x, r) - 1 \vert > 1 \Big] \label{eq:reduce_stability_bound}\,.
% \end{align}
We want to show that $\p[\vert \hat{C}_k(x) - \hat{C}(x) \vert < \epsilon] \to 1$ as $k \to \infty$. By equation~\eqref{eq:yhat_near_1}, it suffices to show that $\vert \hat{C}_k(x) - \hat{C}(x) \vert < \epsilon$ for sufficiently large $k$ if
\begin{equation}\label{eq:ynear1_assumption}
    \max_{r \in R_k} \vert \hat{y}(x, r) - 1 \vert \leq 1\,.
\end{equation}
Therefore, for the remainder of the proof, we assume equation~\eqref{eq:ynear1_assumption} holds.

First, we obtain an upper bound on $\hat{C}_k(x) - \hat{C}(x)$. The upper bounds in equations~\eqref{eq:j>=2_assumption}--\eqref{eq:j=1_assumption} imply that
\begin{align*}
    \hat{C}_k(x) &= \frac{5}{\rmaxk^5 - \rmink^5} \Bigg(\sum_{j=1}^{m_k} r_j^2\Big(\hat{y}_k(x, r_j) - 1\Big) \Drk \Bigg) \\
    &\leq \frac{5}{\rmaxk^5 - \rmink^5} \Bigg(\sum_{j=1}^{m_k} r_j^2 \Big(\hat{y}(x, r_{j + \ell_k}) - 1 + \xi_k\Big) \Drk\Bigg)\,.
\end{align*}
Substituting $r_j^2 = r_{j + \ell_k}^2 - \ell_k \Drk (2r_j + \ell_k \Drk)$, we obtain
\begin{align*}
    \hat{C}_k(x) &\leq \frac{5}{\rmaxk^5 - \rmink^5} \Bigg(\sum_{j=1}^{m_k} \Big(r_{j+\ell_k}^2 - \ell_k \Drk (2r_j + \ell_k \Drk)\Big) \Big(\hat{y}(x, r_{j + \ell_k}) - 1 + \xi_k\Big) \Drk\Bigg) \\
	% &= \frac{5}{\rmaxk^5 - \rmink^5} \Bigg( \sum_{j=1}^{m_k}r_{j+\ell_k}^2 \Big(\hat{y}(x, r_{j + \ell_k}) - 1 + \xi_k \Big)\Drk \\
	% &\qquad - \sum_{j=1}^{m_k} \Big( \ell_k \Drk (2r_j + \ell_k \Drk) \Big)\Big(\hat{y}(x, r_{j + \ell_k}) - 1 + \xi_k\Big) \Drk \Bigg) \\
	&= \frac{5}{\rmaxk^5 - \rmink^5} \Bigg( \sum_{j=1 + \ell_k}^{m_k + \ell_k}r_j^2 \Big(\hat{y}(x, r_j) - 1 + \xi_k \Big)\Drk \\
	&\qquad - \sum_{j=1}^{m_k} \Big( \ell_k \Drk (2r_j + \ell_k \Drk) \Big)\Big(\hat{y}(x, r_{j + \ell_k}) - 1 + \xi_k\Big) \Drk \Bigg)\,.
\end{align*}
Rearranging terms, we obtain
\begin{align*}
    \hat{C}_k(x) &\leq \hat{C}(x) + \frac{5}{\rmaxk^5 - \rmink^5} \Bigg(\sum_{j = 1 + \ell_k}^{m_k + \ell_k} r_j^2 \xi_k \Drk + \sum_{j= m_k +1}^{m_k+\ell_k} r_j^2\Big(\hat{y}(x, r_j) - 1\Big)\Drk \\
    &\qquad\qquad\qquad\qquad -  \sum_{j=1}^{\ell_k} r_j^2\Big(\hat{y}(x, r_j) - 1 \Big)\Drk \\
    &\qquad\qquad\qquad\qquad - \sum_{j=1}^{m_k} \ell_k \Drk (2r_j + \ell_k\Drk)\Big(\hat{y}(x, r_{j + \ell_k}) - 1 + \xi_k \Big)\Drk\Bigg)\,.
\end{align*}
By equation~\eqref{eq:ynear1_assumption},
\begin{align*}
    \hat{C}_k(x) &\leq \hat{C}(x) + \frac{5\xi_k}{\rmaxk^5 - \rmink^5}\sum_{j= 1 + \ell_k}^{m_k + \ell_k} r_j^2 \Drk \\
    &+ \frac{5(1 + \xi_k)}{\rmaxk^5 - \rmink^5} \Bigg(\sum_{j= m_k + 1}^{m_k+\ell_k} r_j^2\Drk + \sum_{j=1}^{\ell_k} r_j^2\Drk
    + \sum_{j=1}^{m_k} \ell_k \Drk (2r_j + \ell_k\Drk)\Drk\Bigg)\,.
\end{align*}
By comparing the sum $\sum_{j= 1 + \ell_k}^{m_k + \ell_k} r_j^2 \Drk$ to the integral $\int_0^{\rmaxk+ \Drk + \delta_k} r^2dr$, we obtain
\begin{align*}
	\hat{C}_k(x) \leq \hat{C}(x) &+ \frac{5\xi_k}{\rmaxk^5 - \rmink^5} \cdot \frac{(\rmaxk+ \Drk + \delta_k)^3}{3} \\
	&+ \frac{5(1 + \xi_k)}{\rmaxk^5 - \rmink^5} \Bigg(\sum_{j= m_k +1}^{m_k+\ell_k} r_j^2\Drk + \sum_{j=1}^{\ell_k} r_j^2\Drk \\
    &\qquad\qquad\qquad\qquad\qquad\qquad + \sum_{j=1}^{m_k} \ell_k \Drk (2r_j + \ell_k\Drk)\Drk\Bigg)\,.
\end{align*}
By hypothesis, $\rmink \leq B \rmaxk$ for some $B < 1$ and $\Drk + \delta_k < \rmaxk$ for sufficiently large $k$, so
\begin{align*}
	\hat{C}_k(x) \leq \hat{C}(x) &+ \frac{40}{3(1 - B^5)} \cdot \frac{\xi_k}{\rmaxk^2} \\
	&+ \frac{5(1 + \xi_k)}{(1 -B^5)\rmaxk^5} \Bigg(\sum_{j= m_k +1}^{m_k+\ell_k} r_j^2\Drk + \sum_{j=1}^{\ell_k} r_j^2\Drk \\
 &\qquad\qquad\qquad\qquad\qquad\qquad + \sum_{j=1}^{m_k} \ell_k \Drk (2r_j + \ell_k\Drk)\Drk\Bigg)\,.
\end{align*}
Because $r_j$ increases monotonically with $j$,
\begin{align*}
		\hat{C}_k(x) &\leq \hat{C}(x) + \frac{40}{3(1 - B^5)} \cdot \frac{\xi_k}{\rmaxk^2} + \frac{5(1 + \xi_k)}{(1 -B^5)\rmaxk^5} \Bigg(\ell_k \Drk (\rmaxk+ \ell_k \Drk)^2 \\
  &\qquad + \ell_k \Drk (\rmink + \ell_k \Drk)^2 + (\rmaxk- \rmin)\ell_k \Drk (2\rmaxk+ \ell_k \Drk)\Bigg)\,.
\end{align*}
By choice of $\ell_k$, we have $\ell_k \Drk < \delta_k + \Drk$, so
\begin{align*}
    \hat{C}_k(x) &\leq \hat{C}(x)  + \frac{40}{3(1 - B^5)} \cdot \frac{\xi_k}{\rmaxk^2} + \frac{5(1 + \xi_k)(\delta_k + \Drk)}{(1 -B^5)\rmaxk^5} \Bigg( (\rmaxk+ (\delta_k + \Drk))^2 \\
    &\qquad + (\rmink + (\delta_k + \Drk))^2 + (\rmaxk- \rmink)(2\rmaxk+ (\delta_k + \Drk))\Bigg)\,.
\end{align*}
Because $0 \leq \rmink < \rmaxk$,
\begin{align*}
	\hat{C}_k(x) &\leq \hat{C}(x)  + \frac{40}{3(1 - B^5)} \cdot \frac{\xi_k}{\rmaxk^2}  + \frac{5(1 + \xi_k)(\delta_k + \Drk)}{(1 -B^5)\rmaxk^5} \Bigg(2(\rmaxk+ (\delta_k + \Drk))^2  \\
 &\qquad + \rmaxk (2\rmaxk+ \delta_k + \Drk)\Bigg)\,.
\end{align*}
By hypothesis, $\delta_k + \Drk < \rmaxk$ for sufficiently large $k$, so
\begin{equation*}
    \hat{C}_k(x) - \hat{C}(x) \leq  \frac{40}{3(1 - B^5)} \cdot \frac{\xi_k}{\rmaxk^2} + \frac{55(1 + \xi_k)}{(1 - B^5)} \cdot \frac{(\delta_k + \Drk)}{\rmaxk^3}\,
\end{equation*}
for sufficiently large $k$. The righthand side is positive and approaches $0$ as $k \to \infty$, so
\begin{equation}\label{eq:lowerS_prob}
    \hat{C}_k(x) - \hat{C}(x) < \epsilon
\end{equation}
for sufficiently large $k$.

Next, we obtain a lower bound on $\hat{C}_k(x) - \hat{C}(x)$. The calculation proceeds almost the same way as our calculation of an upper bound, except that the lower bound in equation~\eqref{eq:j=1_assumption} is of a slightly different form than the upper bound. The lower bounds in equations~\eqref{eq:j>=2_assumption}--\eqref{eq:j=1_assumption} imply that
\begin{align*}
	\hat{C}_k(x) &= \frac{5}{\rmaxk^5 - \rmink^5} \Bigg(\sum_{j=1}^{m_k} r_j^2\Big(\hat{y}_k(x, r_j) - 1\Big) \Drk \Bigg) \\
	&\geq \frac{5}{\rmaxk^5 - \rmink^5} \Bigg( r_1^2\Big(\hat{y}_k(x, r_1 - \delta_k) -1 -  \xi_k \Big)\Drk \\
 &\qquad + \sum_{j=2}^{m_k} r_j^2 \Big(\hat{y}(x, r_{j - \ell_k}) - 1 - \xi_k\Big) \Drk \Bigg)\,.
\end{align*}
Substituting $r_j^2 = r_{j - \ell_k}^2 + \ell_k \Drk (2r_j - \ell_k \Drk)$, we obtain
\begin{align*}
	\hat{C}_k(x) &\geq \frac{5}{\rmaxk^5 - \rmink^5} \Bigg( r_1^2\Big(\hat{y}(x, r_1 - \delta_k) -1 -  \xi_k \Big)\Drk \\
 &\qquad + \sum_{j=2 - \ell_k}^{m_k-\ell_k} r_j^2\Big(\hat{y}(x, r_j) - 1 - \xi_k \Big) \Drk \\
	&\qquad + \sum_{j=2}^{m_k} \ell_k \Drk (2r_j - \ell_k \Drk) \Big(\hat{y}(x, r_{j - \ell_k}) - 1 - \xi_k \Big) \Drk\Bigg)\,.
\end{align*}
Rearranging terms, we have
\begin{align*}
	\hat{C}_k(x) &\geq \hat{C}(x) + \frac{5}{\rmaxk^5 - \rmink^5} \Bigg(r_1^2\Big(\hat{y}(x, r_1 - \delta_k) -1 -  \xi_k \Big)\Drk \\
 &\qquad + \sum_{j = 2- \ell_k}^0 r_j^2\Big(\hat{y}(x, r_j) - 1\Big)\Drk
	- r_1^2\Big(\hat{y}(x, r_1) - 1\Big) \Drk \\
 &\qquad - \sum_{j = m_k - \ell_k + 1}^{m_k} r_j^2 \Big(\hat{y}(x, r_j) - 1\Big)\Drk
 - \sum_{j = 2 - \ell_k}^{m_k - \ell_k} r_j^2 \xi_k \Drk \\
	&\qquad + \sum_{j=2}^{m_k} \ell_k \Drk (2r_j - \ell_k \Drk) \Big(\hat{y}(x, r_{j - \ell_k}) - 1 - \xi_k \Big) \Drk\Bigg)\,.
\end{align*}
By equation~\eqref{eq:ynear1_assumption},
\begin{align*}
	\hat{C}_k(x) &\geq \hat{C}(x) - \frac{5\xi_k}{\rmaxk^5 - \rmink^5}\sum_{j = 2 - \ell_k}^{m_k - \ell_k} r_j^2 \Drk 
	- \frac{5(1 + \xi_k)}{\rmaxk^5 - \rmink^5} \Bigg( r_1^2\Drk 
	+ \sum_{j = 2- \ell_k}^1 r_j^2 \Drk\\
 &\qquad + \sum_{j = m_k - \ell_k + 1}^{m_k} r_j^2 \Drk
	+ \sum_{j=2}^{m_k} \ell_k \Drk (2r_j - \ell_k \Drk)  \Drk\Bigg)\,.
\end{align*}
By comparing the sum $\sum_{j = 2 - \ell_k}^{m_k - \ell_k} r_j^2 \Drk$ to the integral $\int_0^{\rmax}r^2 dr$, we obtain
\begin{align*}
	\hat{C}_k(x) &\geq \hat{C}(x) - \Bigg( \frac{5\xi_k}{\rmaxk^5 - \rmink^5} \cdot \frac{\rmaxk^3}{3}\Bigg) 
	- \frac{5(1 + \xi_k)}{\rmaxk^5 - \rmink^5}\Bigg( r_1^2\Drk 
	+ \sum_{j = 2- \ell_k}^1 r_j^2 \Drk \\
 &\qquad + \sum_{j = m_k - \ell_k + 1}^{m_k} r_j^2 \Drk
	+ \sum_{j=2}^{m_k} \ell_k \Drk (2r_j - \ell_k \Drk)  \Drk\Bigg)\,.
\end{align*}
By hypothesis, $\rmink \leq B \rmaxk$ for some $B < 1$, so
\begin{align*}
	\hat{C}_k(x) &\geq \hat{C}(x) - \Bigg( \frac{5}{3(1 - B^5)} \cdot \frac{\xi_k}{\rmaxk^2}\Bigg)
	- \frac{5(1 + \xi_k)}{(1 - B^5)\rmaxk^5}\Bigg( r_1^2\Drk 
	+ \sum_{j = 2- \ell_k}^1 r_j^2 \Drk\\
	&\qquad + \sum_{j = m_k - \ell_k + 1}^{m_k} r_j^2 \Drk + \sum_{j=2}^{m_k} \ell_k \Drk (2r_j - \ell_k \Drk)  \Drk\Bigg)\,.
\end{align*}
Because $r_j^2$ increases monotonically with $j \in \{2 - \ell_k, \ldots, m_k\}$ (and noting that $r_{2 - \ell_k} > 0$ by hypothesis),
\begin{align*}
	\hat{C}_k(x) &\geq \hat{C}(x) - \Bigg( \frac{5}{3(1 - B^5)} \cdot \frac{\xi_k}{\rmaxk^2}\Bigg) - \frac{5(1 + \xi_k)}{(1 - B^5)\rmaxk^5}\Bigg( r_1^2 \Drk 
	+ \ell_k r_1^2\Drk \\
	&\qquad 
	+ \ell_k \rmaxk^2 \Drk + (m_k -1)\Drk\ell_k \Drk (2\rmaxk - \ell_k \Drk)\Bigg)\,. \\
	&\geq \hat{C}(x) - \frac{5}{3(1 - B^5)} \cdot \frac{\xi_k}{\rmaxk^2}
	- \frac{5(1 + \xi_k)}{(1 - B^5)\rmaxk^5}\Bigg( (1+ \ell_k)r_1^2 \Drk 
	+ 3 \rmaxk^2 \Drk \ell_k\Bigg) \\
	&\geq \hat{C}(x) - \frac{5}{3(1 - B^5)} \cdot \frac{\xi_k}{\rmaxk^2} \\
	&\qquad - \frac{5(1 + \xi_k)}{(1 - B^5)\rmaxk^5}\Bigg( (1+ 2\ell_k)\rmaxk^2 \Drk
+ 3 \rmaxk^2 \Drk \ell_k\Bigg)\,. \\
\end{align*}
By choice of $\ell_k$, we have $\ell_k \Drk < \Drk + \delta_k$, which implies
\begin{align*}
\hat{C}_k(x) - \hat{C}(x) &\geq  - \Bigg( \frac{5}{3(1 - B^5)} \cdot \frac{\xi_k}{\rmaxk^2}\Bigg) - \frac{5(1 + \xi_k)}{(1 - B^5)\rmaxk^3}\Big(6\Drk + 5 \delta_k \Big)\\
&\geq
- \Bigg( \frac{5}{3(1 - B^5)} \cdot \frac{\xi_k}{\rmaxk^2}\Bigg) - \frac{30(1 + \xi_k)}{(1 - B^5)} \cdot \frac{\Drk + \delta_k}{\rmaxk^3}\,.
\end{align*}
The righthand side is negative and approaches $0$ as $k \to \infty$, so 
\begin{equation}\label{eq:upperS_prob}
    \hat{C}_k(x) - \hat{C}(x) > - \epsilon
\end{equation}
for sufficiently large $k$. Together, equations~\eqref{eq:lowerS_prob} and~\eqref{eq:upperS_prob} complete the proof.
\end{proof}

\begin{lemma}\label{lem:Cconverge}
If the hyperparameter value sequences satisfy
\begin{enumerate}
    \item $\Drk/\rmaxk^3 \to 0$ as $k \to \infty$\,,
    \item $|X_k|(\rmink + \Drk)^n \to \infty$\,, and
    \item $\rmink/\rmaxk^3 \to 0$ as $k \to \infty$\,,
\end{enumerate}
then $|\hat{C}[\dkexact, \rho](x_k) - C(x_k)| \to 0$ in probability as $k \to \infty$, where $\{x_k\}$ is any sequence of points such that $x_k \in X_k$.
\end{lemma}
\begin{proof}
Let $x$ be any point in $X_k$. For all $j \in \{1, \ldots, m_k\}$, let $\hat{y}_j := \hat{y}(x, r_j)$ and let $y_j := y(x, r_j)$. The absolute difference $\vert \hat{C}(x) - C(x)\vert$ is bounded above by
\begin{align*}
    \vert \hat{C}(x) - C(x)\vert &\leq \frac{5}{\rmaxk^5 - \rmink^5} \Bigg( 
    \Big\vert \sum_{j=1}^{m_k} r_j^2 \Drk - \int_{\rmink}^{\rmaxk} r^2 dr\Big\vert \notag \\
    &\qquad  
    + \Big\vert \int_{\rmink}^{\rmaxk} r^2 y(x, r)dr - \sum_{j=1}^{m_k} r_j^2 y_j \Drk\Big\vert 
    + \Big\vert \sum_{j=1}^{m_k} r_j^2 (\hat{y}_j - y_j)\Drk\Big\vert\Bigg)\,.
\end{align*}
Because $\rmink/\rmaxk^3 \to 0$ (by hypothesis), there is a constant $B < 1$ such that $\rmink \leq B \rmaxk$ for all $k$. Therefore,
\begin{align}
 \vert \hat{C}(x) - C(x)\vert &\leq \frac{5}{(1 - B^5)\rmaxk^5} \Bigg( 
    \Big\vert \sum_{j=1}^{m_k} r_j^2 \Drk - \int_{\rmink}^{\rmaxk} r^2 dr\Big\vert \notag \\
    &\qquad  
    + \Big\vert \int_{\rmink}^{\rmaxk} r^2 y(x, r)dr - \sum_{j=1}^{m_k} r_j^2 y_j \Drk\Big\vert 
    + \Big\vert \sum_{j=1}^{m_k} r_j^2 (\hat{y}_j - y_j)\Drk\Big\vert\Bigg)\,.\label{eq:initialCbound}
\end{align}
The first term on the righthand side of equation~\eqref{eq:initialCbound} is a Riemann sum error. For any function $f(r)$ that is integrated on $[\rmink, \rmaxk]$, the error in the right Riemann sum is bounded above by $\max_{r \in [\rmink, \rmaxk]} \vert f'(r) \vert \Drk \cdot (\rmaxk - \rmink)/2$. Therefore,
\begin{align}
    &\Big\vert \int_{\rmink}^{\rmaxk} r^2 y(x, r)dr - \sum_{j=1}^{m_k} r_j^2 y_j \Drk\Big\vert \notag \\ 
    &\qquad \leq \Drk \Big(\max_{r \in [\rmink, \rmaxk]} \Big\vert \frac{d}{dr}r^2 y(x, r)\Big\vert \Big)(\rmaxk - \rmink)/2 \notag \\
    &\qquad \leq \Drk \Big(\max_{r \in [\rmink, \rmaxk]} \Big\vert \frac{d}{dr}r^2 y(x, r)\Big\vert \Big) \rmaxk/2.\label{eq:term2bound}
\end{align}
We have $\frac{d}{dr}r^2 y(x, r) = r^2 \frac{d}{dr} y(x, r) + 2r y(x, r)$. By equation~\eqref{eq:scalar_ball}, we have $\lim_{r \to 0} y(x, r) = 1$ and $\lim_{r \to 0} \frac{d}{dr} y(x, r) = 0$. Therefore, $\vert \frac{d}{dr}r^2 y(x, r)\vert = \mathcal{O}(r)$ as $r \to 0$, so there is a constant $A > 1$ such that 
\[
\max_{r \in [\rmink, \rmaxk]}\vert \frac{d}{dr}r^2 y(x, r)\vert \leq 2 A \rmaxk
\]
for sufficiently small $\rmaxk$. Thus for sufficiently large $k$,
\begin{equation}\label{eq:derivative_bound}
\Big(\max_{r \in [\rmink, \rmaxk]} \Big\vert \frac{d}{dr}r^2 y(x, r)\Big\vert\Big) \rmaxk/2 \leq A\rmaxk^2
\end{equation}
because $\rmaxk \to 0$ as $k \to \infty$. 

Let $\epsilon > 0$. By hypothesis, $\frac{\rmaxk^3}{\Drk} \to \infty$ as $k \to \infty$, so 
\begin{equation}\label{eq:rmax_squared}
\rmaxk^2 \leq \frac{\rmaxk^5(1-B^5)\epsilon}{15A \Drk}
\end{equation}
for sufficiently large $k$. Substituting equation~\eqref{eq:rmax_squared} into equation~\eqref{eq:derivative_bound} and equation~\eqref{eq:derivative_bound} into equation~\eqref{eq:term2bound} yields
\begin{equation}\label{eq:term2bound_final}
 \Big\vert \int_{\rmink}^{\rmaxk} r^2 y(x, r)dr - \sum_{j=1}^{m_k} r_j^2 y_j \Drk\Big\vert \leq \frac{\rmaxk^5(1-B^5)\epsilon}{15}\,.
\end{equation}

Next, we bound the second term on the righthand side of equation~\eqref{eq:initialCbound}, which is also a Riemann sum error. For a monotonic function $f(r)$ that is integrated on $[\rmink, \rmaxk]$, the error in the right Riemann sum is bounded above by $\Drk \vert f(\rmaxk) - f(\rmink)\vert$. Therefore,
\begin{equation*}
    \Big\vert \sum_{j=1}^{m_k} r_j^2 \Drk - \int_{\rmink}^{\rmaxk} r^2 dr\Big\vert \leq \Drk (\rmaxk^2 - \rmink^2) \leq \Drk \cdot \rmaxk^2\,.
\end{equation*}
By Eq. \eqref{eq:rmax_squared},
\begin{equation}\label{eq:term1bound_final}
 \Big\vert \sum_{j=1}^{m_k} r_j^2 \Drk - \int_{\rmink}^{\rmaxk} r^2 dr\Big\vert \leq \frac{\rmaxk^5(1-B^5)\epsilon}{15A} < \frac{\rmaxk^5(1-B^5)\epsilon}{15}
\end{equation}
for sufficiently large $k$.

Putting the inequalities of Eqs. \eqref{eq:term1bound_final} and \eqref{eq:term2bound_final} into Eq. \eqref{eq:initialCbound}, we obtain
\begin{equation*}
    \vert \hat{C}(x) - C(x)\vert \leq \frac{2}{3}\epsilon + \Big\vert \sum_{j=1}^{m_k} r_j^2(\hat{y}_j - y_j) \Big\vert \frac{5 \Drk}{(1 - B^5) \rmaxk^5}.
\end{equation*}
Therefore,
\begin{equation}\label{eq:Cprobbound1}
    \p[ \vert \hat{C}(x) - C(x) \vert > \epsilon ] \leq \p\Bigg[ \Big\vert \sum_{j=1}^{m_k} r_j^2 (\hat{y}_j -y_j) \Big\vert > \frac{(1-B^5)\rmaxk^5 \epsilon}{15 \Drk}\Bigg]\,.
\end{equation}
We have $\E\Big[\sum_{j=1}^{m_k} r_j^2\hat{y}_j\Big] = \sum_{j=1}^{m_k} r_j^2 y_j$ because $\E[\hat{y}_j] = y_j$ (Lemma \ref{lem:E_volhat}). By applying Chebyshev's inequality to the righthand side of Eq. \eqref{eq:Cprobbound1}, we obtain
\begin{equation}\label{eq:cheby}
    \p[\vert \hat{C}(x) - C(x) \vert > \epsilon] \leq \Big(\frac{15}{(1 - B^5)\epsilon}\Big)^2 \Big(\frac{\Drk}{\rmaxk^5} \Big)^2 \var \Big( \sum_{j=1}^{m_k} r_j^2 \hat{y}_j \Big).
\end{equation}
We expand the variance as
\begin{equation*}
    \var \Big( \sum_{j=1}^{m_k} r_j^2 \hat{y}_j \Big) = \sum_{j=1}^{m_k} r_j^4 \var(\hat{y}_j) + \sum_{i \neq j} r_i^2 r_j^2 \cov(\hat{y}_i, \hat{y}_j).
\end{equation*}
For all $i \neq j$, we have $\cov(\hat{y}_i, \hat{y}_j)^2 \leq \var(\hat{y}_i) \var(\hat{y}_j)$. Therefore,
\begin{equation*}
    \var \Big( \sum_{j=1}^{m_k} r_j^2 \hat{y}_j \Big) \leq \Big( \sum_{j=1}^{m_k} r_j^2 \sqrt{\var(\hat{y}_j)} \Big)^2.
\end{equation*}
By Proposition~\ref{prop:MSE}, there is a constant $A' \geq 0$ such that
\begin{equation*}
    \var(\hat{y}_j) \leq \frac{A'}{|X_k|r_j^n}
\end{equation*}
for all $j$ and sufficiently large $k$. Therefore,
\begin{equation}\label{eq:varbound}
    \var \Big( \sum_{j=1}^{m_k} r_j^2 \hat{y}_j \Big) \leq \frac{A'}{|X_k|}\Big(\sum_{j=1}^{m_k} r_j^{2 - n/2}\Big)^2
\end{equation}
for sufficiently large $k$. Below, we use Eq. \eqref{eq:varbound} to obtain an upper bound on the righthand side of Eq. \eqref{eq:cheby}. There are two cases, depending on $n$.

\vspace{5mm}

\noindent{\bf Case 1}: ($2 \leq n \leq 4$).

In this case,
\begin{equation}\label{eq:sumboundcase1}
    \Big(\sum_{j=1}^{m_k} r_j^{2 - n/2}\Big)^2 \leq \rmaxk^{4 - n} \Big( \frac{\rmaxk - \rmink}{\Drk}\Big)^2 \leq \frac{\rmaxk^{6- n}}{\Drk^2}
\end{equation}
because $r_j^{2 - n/2}$ is monotonically increasing. Combining Eqs \eqref{eq:cheby}, \eqref{eq:varbound}, and \eqref{eq:sumboundcase1}, we obtain
\begin{align*}
	\p[\vert \hat{C}(x) - C(x) \vert > \epsilon] &\leq \Big(\frac{15}{(1-B^5)\epsilon}\Big)^2 \Big(\frac{\Drk}{\rmaxk^5}\Big)^2\frac{A'}{|X_k|} \frac{\rmaxk^{6-n}}{\Drk^2} \\
&\leq A'\Big(\frac{15}{(1-B^5)\epsilon}\Big)^2 \frac{1}{|X_k| \rmaxk^{n+4}} \\
&= A'\Big(\frac{15}{(1-B^5)\epsilon}\Big)^2 \frac{1}{|X_k|(\rmink + \Drk)^{n/3 + 4/3} }\Big(\frac{\rmink + \Drk}{\rmaxk^3}\Big)^{n/3 + 4/3}\,.
\end{align*}
Because $n/3 + 4/3 \leq n$ for $n \geq 2$ and $\rmink + \Drk < 1$ for sufficiently large $k$,
\begin{equation*}
\p[\vert \hat{C}(x) - C(x) \vert > \epsilon] \leq A'\Big(\frac{15}{(1-B^5)\epsilon}\Big)^2 \frac{1}{|X_k|(\rmink + \Drk)^n}\Big(\frac{\rmink + \Drk}{\rmaxk^3}\Big)^{n/3 +3/4}
\end{equation*}
for sufficiently large $k$. By hypothesis, $\Big(\frac{\rmink + \Drk}{\rmaxk^3}\Big) \to 0$ and $\fracc{|X_k|(\rmink + \Drk)^n} \to 0$ as $k \to \infty$. Therefore, $\p[ \vert \hat{C}(x_k) - C(x_k) \vert > \epsilon] \to 0$ as $k \to \infty$.

\vspace{5mm}

\noindent {\bf Case 2:} ($n > 4$).

In this case,
\begin{equation}\label{eq:sumboundcase2}
\Big(\sum_{j=1}^{m_k} r_j^{2 - n/2}\Big)^2 \leq (\rmink + \Drk)^{4 - n}\Big(\frac{\rmaxk - \rmink}{\Drk}\Big)^2 \leq (\rmink + \Drk)^{4 - n}\Big(\frac{\rmaxk}{\Drk}\Big)^2
\end{equation}
because $r_j^{2 - n/2}$ is monotonically decreasing. Combining Eqs \eqref{eq:cheby}, \eqref{eq:varbound}, and \eqref{eq:sumboundcase2} yields
\begin{align*}
	\p[\vert \hat{C}(x) - C(x)\vert > \epsilon] &\leq A'\Big(\frac{15}{(1-B^5)\epsilon}\Big)^2 \Big(\frac{\Drk}{\rmaxk^5}\Big)^2 \frac{(\rmink + \Drk)^{4 - n} \rmaxk^2}{|X_k|\Drk^2} \\
	% &= A'\Big(\frac{15}{(1-B^5)\epsilon}\Big)^2 \frac{(\rmink + \Drk)^4}{|X_k|\rmaxk^8 (\rmink + \Drk)^n}\\
 &= A'\Big(\frac{15}{(1-B^5)\epsilon}\Big)^2 \Big(\frac{\rmink + \Drk}{\rmaxk^2}\Big)^4 \fracc{|X_k| (\rmink + \Drk)^n}\,.
\end{align*}
By hypothesis, $\Big(\frac{\rmink + \Drk}{\rmaxk^2}\Big) \to 0$ and $\fracc{|X_k| (\rmink + \Drk)^n} \to 0$ as $k \to \infty$. Therefore, $\p[\vert \hat{C}(x_k) - C(x_k) \vert > \epsilon ] \to 0$ as $k \to \infty$.
\end{proof}


\begin{thebibliography}{10}
	
	\bibitem{arvan}
	G.~Arvanitidis, S.~Hauberg, P.~Hennig, and M.~Schober.
	\newblock Fast and robust shortest paths on manifolds learned from data.
	\newblock In {\em Proceedings of the 22nd International Conference on
		Artificial Intelligence and Statistics (AISTATS)}, volume~89, pages
	1506--1515. PMLR, 2019.
	
	\bibitem{graph_approx}
	M.~Bernstein, V.~de~Silva, J.~C. Langford, and J.~B. Tenenbaum.
	\newblock Graph approximations to geodesics on embedded manifolds.
	\newblock 2000.
	
	\bibitem{smita}
	Dhananjay Bhaskar, Kincaid MacDonald, Oluwadamilola Fasina, Dawson~S. Thomas,
	Bastian Rieck, Ian Adelstein, and Smita Krishnaswamy.
	\newblock Diffusion curvature for estimating local curvature in high
	dimensional data.
	\newblock In Alice~H. Oh, Alekh Agarwal, Danielle Belgrave, and Kyunghyun Cho,
	editors, {\em Advances in Neural Information Processing Systems}, 2022.
	
	\bibitem{BHV}
	Louis~J. Billera, Susan~P. Holmes, and Karen Vogtmann.
	\newblock Geometry of the space of phylogenetic trees.
	\newblock {\em Advances in Applied Mathematics}, 27(4):733--767, 2001.
	
	\bibitem{bubenik}
	Peter Bubenik, Michael Hull, Dhruv Patel, and Benjamin Whittle.
	\newblock Persistent homology detects curvature.
	\newblock {\em Inverse Problems}, 36(2):025008, 2020.
	
	\bibitem{dim_est_review}
	Francesco Camastra and Antonino Staiano.
	\newblock Intrinsic dimension estimation: Advances and open problems.
	\newblock {\em Information Sciences}, 328:26--41, 2016.
	
	\bibitem{osculating}
	Frederic Cazals and Marc Pouget.
	\newblock Estimating differential quantities using polynomial fitting of
	osculating jets.
	\newblock {\em Computer Aided Geometric Design}, 22(2):121--146, 2005.
	
	\bibitem{chazal}
	F.~Chazal, D.~Cohen-Steiner, A.~Lieutier, and B.~Thibert.
	\newblock Stability of curvature measures.
	\newblock {\em Computer Graphics Forum}, 28(5):1485--1496, 2009.
	
	\bibitem{pcpnet}
	Paul Guerrero, Yanir Kleiman, Maks Ovsjanikov, and Niloy~J. Mitra.
	\newblock {PCPNet}: Learning local shape properties from raw point clouds.
	\newblock {\em Computer Graphics Forum}, 37(2):75--85, 2018.
	
	\bibitem{poincare}
	A.~Klimovskaia, D.~Lopez-Paz, L.~Bottou, and M.~Nickel.
	\newblock Poincar\'{e} maps for analyzing complex hierarchies in single-cell
	data.
	\newblock {\em Nature Communications}, 11:2966, 2020.
	
	\bibitem{bickel}
	Elizaveta Levina and Peter Bickel.
	\newblock Maximum likelihood estimation of intrinsic dimension.
	\newblock In L.~Saul, Y.~Weiss, and L.~Bottou, editors, {\em Advances in Neural
		Information Processing Systems}, volume~17. MIT Press, 2004.
	
	\bibitem{spherelets}
	Didong Li and David~B. Dunson.
	\newblock Geodesic distance estimation with spherelets.
	\newblock {\em arxiv:1907.00296}, 2019.
	
	\bibitem{hierarchy}
	M.~Nickel and D.~Kiela.
	\newblock Poincar\'{e} embeddings for learning hierarchical representations.
	\newblock In I.~Guyon, U.~Von Luxburg, S.~Bengio, H.~Wallach, R.~Fergus,
	S.~Vishwanathan, and R.~Garnett, editors, {\em Advances in Neural Information
		Processing Systems}, volume~30. Curran Associates, Inc., 2017.
	
	\bibitem{kde_submanifold}
	Arkadas Ozakin and Alexander Gray.
	\newblock Submanifold density estimation.
	\newblock In Y.~Bengio, D.~Schuurmans, J.~Lafferty, C.~Williams, and
	A.~Culotta, editors, {\em Advances in Neural Information Processing Systems},
	volume~22. Curran Associates, Inc., 2009.
	
	\bibitem{Rman_stats}
	Xavier Pennec.
	\newblock Intrinsic statistics on {R}iemannian manifolds: Basic tools for
	geometric measurements.
	\newblock {\em Journal of Mathematical Imaging and Vision}, 25:127, 2006.
	
	\bibitem{petersen}
	Peter Petersen.
	\newblock {\em Riemannian Geometry}, volume 171 of {\em Graduate Texts in
		Mathematics}.
	\newblock Springer, New York, NY, 2nd edition, 2006.
	
	\bibitem{pointnet}
	Charles~R. Qi, Hao Su, Kaichun Mo, and Leonidas~J. Guibas.
	\newblock Pointnet: Deep learning on point sets for 3d classification and
	segmentation.
	\newblock In {\em Proceedings of the IEEE Conference on Computer Vision and
		Pattern Recognition (CVPR)}, 2017.
	
	\bibitem{tradeoffs}
	F.~Sala, C.~{De Sa}, A.~Gu, and C.~Re.
	\newblock Representation tradeoffs for hyperbolic embeddings.
	\newblock In J.~Dy and A.~Krause, editors, {\em Proceedings of the 35th
		International Conference on Machine Learning}, volume~80 of {\em Proceedings
		of Machine Learning Research}, pages 4460--4469. PMLR, 2018.
	
	\bibitem{network_curvature_review}
	Areejit Samal, R.~P. Sreejith, Jiao Gu, Shiping Liu, Emil Saucan, and Jürgen
	Jost.
	\newblock Comparative analysis of two discretizations of ricci curvature for
	complex networks.
	\newblock {\em Scientific Reports}, 8(8650), 2018.
	
	\bibitem{curvature_cancer}
	Romeil Sandhu, Tryphon Georgiou, Ed~Reznik, Liangjia Zhu, Ivan Kolesov, Yasin
	Senbabaoglu, and Allen Tannenbaum.
	\newblock Graph curvature for differentiating cancer networks.
	\newblock {\em Scientific Reports}, 5(12323), 2015.
	
	\bibitem{other_scalar}
	R.P. Sreejith, Jürgen Jost, Emil Saucan, and Areejit Samal.
	\newblock Systematic evaluation of a new combinatorial curvature for complex
	networks.
	\newblock {\em Chaos, Solitons \& Fractals}, 101:50--67, 2017.
	
	\bibitem{pnas}
	Duluxan Sritharan, Shu Wang, and Sahand Hormoz.
	\newblock Computing the {R}iemannian curvature of image patch and single-cell
	{RNA} sequencing data manifolds using extrinsic differential geometry.
	\newblock {\em Proceedings of the National Academy of Sciences},
	118(29):e2100473118, 2021.
	
	\bibitem{isomap}
	Joshua~B. Tenenbaum, Vin {de Silva}, and John~C. Langford.
	\newblock A global geometric framework for nonlinear dimensionality reduction.
	\newblock {\em Science}, 290(5500):2319--2323, 2000.
	
\end{thebibliography}
\end{document}